\documentclass{colt2026} %
\makeatletter
\@ifundefined{todo}{}{%
  \let\MM@class@todo\todo
  \let\todo\relax
}
\makeatother

\usepackage[textsize=tiny,disable]{todonotes}
\newcommand{\todov}[2][]{\xspace}
\newcommand{\todoc}[2][]{\xspace}

\usepackage{aliascnt}

\usepackage{xspace}
\usepackage{multirow}
\usepackage{booktabs}
\usepackage{caption}
\usepackage{graphicx}

\newaliascnt{lemma}{theorem}
\newtheorem{lemma}[lemma]{Lemma}
\aliascntresetthe{lemma}

\newaliascnt{corollary}{theorem}
\newtheorem{corollary}[corollary]{Corollary}
\aliascntresetthe{corollary}

\newaliascnt{definition}{theorem}
\newtheorem{definition}[definition]{Definition}
\aliascntresetthe{definition}

\usepackage[capitalize,noabbrev,nameinlink]{cleveref}
\setcitestyle{square}
\let\cref\Cref

\usepackage{theorems}          %
\usepackage{mathmacros}        %
\usepackage{projectmacros}     %

\usepackage{microtype}
\usepackage{enumitem}

\usepackage{tikz}
\usetikzlibrary{arrows.meta,positioning,calc,decorations.pathreplacing}

\title[Trajectory Data Suffices for Statistically Efficient Offline RL]{
  Trajectory Data Suffices for Statistically Efficient Policy Evaluation in Fixed-Horizon Offline RL with Linear $q^\pi$-Realizability and Concentrability
}
\usepackage{times}

\coltauthor{%
 \Name{Volodymyr Tkachuk} \Email{vtkachuk@ualberta.ca}\\
 \addr University of Alberta
 \AND
 \Name{{Cs}aba {Sz}epesv\'ari} \Email{szepesva@ualberta.ca}\\
 \addr University of Alberta
 \AND
 \Name{Xiaoqi Tan} \Email{xiaoqi.tan@ualberta.ca}\\
 \addr University of Alberta
}

\begin{document}

\maketitle

\thispagestyle{plain}

\begin{abstract}%
  We study fixed-horizon offline reinforcement learning (RL) with function approximation for both policy evaluation and policy optimization.
  Prior work established that statistically efficient learning is impossible 
  for either of these problems when the only assumptions are that the data has good coverage (concentrability) 
  and the state-action value function of every policy is linearly realizable ($q^\pi$-realizability) \citep{foster2022offline}.
  Recently, \citet{tkachuk2024trajectory} gave a statistically efficient learner for policy optimization,
  if in addition the data is assumed to be given as trajectories.
  In this work we present a statistically efficient learner for policy evaluation under the same assumptions,
  with the additional requirement that the behavior policy is known.
  Further, we show that the sample complexity of the learner used by \citet{tkachuk2024trajectory} for policy optimization
  can be improved by a tighter analysis.

\end{abstract}


\section{Introduction and Overview of Results} \label{sec:intro}

In offline RL a learner is given access to a dataset and is tasked with either 
evaluating a policy or finding the optimal policy \citep{jiang2024offline}.
If the dataset does not cover the state-action space well, then certain parts of the Markov decision process (MDP) will be unseen, and no learner will be able to succeed \citep{chen2019information}.
Thus, some assumption needs to be made on the coverage of the dataset to get positive results.
One such assumption, which will be central to this work, is when the dataset covers the states and actions reachable by \textbf{all} policies well.
This assumption is known as \textit{concentrability}, with parameter $\conc$ indicating the degree of coverage (see \cref{ass:concentrability}).

When the state space is large it is undesirable for the size of the dataset to scale with the number of states in the MDP.
As such, function approximation is often used to represent particular aspects of the problem.
In this work we focus on the case where a function class is used to represent state-action value functions ($q$-functions).
\citet{foster2022offline} showed that if the dataset satisfies concentrability,
and the $q$-function of \textbf{all} policies is realizable by the function class,
then achieving $\eps$-optimal policy evaluation or optimization with high probability requires a dataset size
that scales with the number of states, 
which is considered statistically inefficient.

Naturally, the question of how much these assumptions need to be strengthened to allow for statistically efficient learning remained.
\citet{tkachuk2024trajectory} showed that if the dataset has the extra structure of being full rollouts from some \textit{behavior policy} (i.e., \textit{trajectories}),
and the $q$-function of every policy is realizable by a linear function class (with known features),
then the sample complexity of policy optimization can be improved to scale polynomially with the feature dimension $d$ and other problem parameters.
Importantly, this is considered statistically efficient, since the feature dimension $d$ is often chosen much smaller than the number of states.
Their result does not imply a similar bound for policy evaluation, since their learner follows a value-iteration approach, which does not evaluate intermediate policies (see \cref{sec:evallearner} for further discussion).
Thus, whether a similar result is possible for policy evaluation remained open.

In this work we give a positive answer,
by providing a statistically efficient learner for policy evaluation under the same assumptions as \citet{tkachuk2024trajectory},
except we additionally assume the behavior policy is known.
Interestingly, \citet{jia2024offline} showed that if 
the $q$-function of \textbf{only} the evaluation policy is realizable by the function class%
,
then the dataset size must scale exponentially with the horizon $H$ of the MDP.
Our result complements their work by showing that if the function class is assumed 
to linearly realize the $q$-function of \textbf{all} memoryless policies,
then the sample complexity can be improved to scale polynomially with the horizon $H$, feature dimension $d$, and other problem parameters.
Recently, \citet{liu2026complexity} showed a negative result for policy optimization when only the $q$-function of the optimal policy is realizable, which complements the result of \citet{jia2024offline} for policy evaluation.
In \cref{table:known results} we summarize our above discussion. 
For a detailed account of the offline RL literature we refer the reader to the excellent survey by \citet{jiang2024offline}.

Another contribution comes as a consequence of our analysis building on the work of \citet{tkachuk2024trajectory}.
In particular, we noticed that a bound in a key lemma used by \citet{tkachuk2024trajectory} can be improved.
Therefore, we also provide a sample complexity bound for the policy optimization objective,
which improves on that of \citet{tkachuk2024trajectory} by a factor of $\conc d$.

The remainder of the paper is structured as follows.
\cref{sec:setting} defines the offline RL setting.
In \cref{sec:results} we present our results.
Then, in \cref{sec:learner} 
we define the learners for which the results hold.

\begin{table}[t]
\centering
\renewcommand{\arraystretch}{1.15} %
\setlength{\aboverulesep}{0pt}     %
\setlength{\belowrulesep}{0pt}     %

\begin{tabular}{lccc} 
    \toprule
    & \multicolumn{2}{c}{Trajectory Data} & \\ 
    \cmidrule(lr){2-3} 
    
    Objective & 
    Obj $\pi$ Realizable & 
    All $\pi$ Lin-Realizable & 
    All $\pi$ Realizable \\
    
    \midrule
    
    Policy Evaluation & 
    \textcolor{red}{x} \citep{jia2024offline} & 
    \textcolor{darkgreen}{\checkmark} \textbf{[This Work]} & 
    \textcolor{red}{x} \citep{foster2022offline} \\
    
    Policy Optimization & 
    \textcolor{red}{x} \citep{liu2026complexity} & 
    \textcolor{darkgreen}{\checkmark} \citep{tkachuk2024trajectory} & 
    \textcolor{red}{x} \citep{foster2022offline} \\
    
    \bottomrule
\end{tabular}

\medskip

\caption{
Concentrability is assumed in all cases.
\textit{Obj $\pi$ Realizable} means \textbf{only} the evaluation/optimal policy is realizable by the function class for policy evaluation/optimization.
\textit{All $\pi$ (Lin-)Realizable} means the $q$-function of \textbf{all} memoryless policies is realizable by the (linear) function class.
The \textcolor{darkgreen}{\checkmark} means a $\poly(d, H, \conc, 1/\eps)$ sample complexity upper bound, 
and \textcolor{red}{x} means a lower bound in terms of the state space size or exponential in the horizon $H$.
}
\label{table:known results}
\end{table}

\section{Problem Setting} \label{sec:setting}
Since we study the same setting as \citet{tkachuk2024trajectory}, we use similar notation and definitions.

\textbf{Notation:} Throughout we fix the integer $d \ge 1$. 
Let $\vec 0 \in \bR^d$ be the $d$-dimensional, all zero vector.
For $L>0$, let $\cB(L)=\{x \in \bR^d: \norm{x}_2 \le L\}$ denote the $d$-dimensional Euclidean ball of radius $L$ centered at the origin, where $\|\cdot\|_2$ denotes the Euclidean norm.
The inner product $\langle x, y \rangle$ for $x, y \in \bR^d$ is defined as the dot product $x^\T y$.
The closure of a set $\cX$ is denoted by $\cl(\cX)$.
Let $\I{B}$ be the indicator function of a boolean-valued variable $B$, taking the value $1$ if $B$ is true and $0$ if false.
Let $\dists(X)$ denote the set of probability distributions over the set $X$.
Let $\bE_{B \sim \cP} [f(B)]$ denote the expectation of a random variable $f(B)$ under distribution $\cP$. 
For integers $i,j$, let $[i]=\{1,\dots,i\}$ and $[i:j]=\{i,\dots,j\}$.
For any two functions $f, g$, comparisons are always pointwise (e.g. $(f \le g)(x) \defeq f(x) \le g(x)$ for all $x$).

The environment is modeled by a fixed-horizon Markov decision process (MDP)
defined by a tuple \((\mathcal{S}, \mathcal{A}, P, \mathcal{R}, H)\). 
The state space \(\mathcal{S}\) is finite but arbitrarily large,
and organized by stages: $\cS = \bigcup_{h \in [H+1]} \cS_h$,
starting from a designated initial state $s_1$ ($\mathcal{S}_1 = \{s_1\}$)\footnote{A deterministic start state $\startstate$ is added for simplicity of presentation. 
It is easy to show that adding an additional stage to the MDP allows for the transition dynamics to encode an arbitrary start state distribution.}
, and culminating in a designated terminal state $\termstate$ ($\cS_{H+1}=\{\termstate\}$)\footnote{A terminal state $\termstate$ is added purely as a technical convenience for the analysis. %
We will focus on the interaction of learners for stages $h \in [H]$ (not $[H+1]$), since the terminal state will have no effect on the learner.}.
Without loss of generality, we assume $\cS_h$ and $\cS_{h'}$ for  $h \neq h'$ are disjoint sets.
The action space \(\mathcal{A}\) is finite. 
The transition kernel is \(P: (\bigcup_{h \in [H]} \mathcal{S}_h) \times \mathcal{A} \to \dists(\mathcal{S})\), with the property that transitions occur between successive stages. 
Specifically, for any \(h \in [H]\), state \(s_h \in \mathcal{S}_h\), and action \(a \in \mathcal{A}\), $P(s_h, a) \in \dists(\cS_{h+1})$.
The reward kernel is $\cR:\cS \times \cA \to \dists([0,1])$. 
To ensure that the terminal state $\termstate$ has no influence on the learner we force the reward kernel to deterministically give zero reward for all actions $a \in \cA$ in $\termstate$.

We will define an agent's interaction with the MDP through a \emph{memoryless policy} $\pi : \cS \to \dists(\cA)$, which assigns a probability distribution over actions based on a state. 
The set of all memoryless policies is $\Pi = \{\pi: \cS \to \dists(\cA)\}$.
For $\pi\in\Pi$, we write $\pi(a|s)$ to denote the probability $\pi(s)$ assigns to action $a$. 
For deterministic policies only (i.e., all the probability is concentrated on a single action) we sometimes abuse notation by writing $\pi(s)$ to denote $\argmax_{a\in\cA} \pi(a|s)$.
An agent interacts with the MDP sequentially from a state $s \in \cS$, 
by sampling an action according to its policy $\pi$, 
receiving a reward specified by $\cR$,
transitioning to a subsequent state according to $P$, 
and repeating this process until receiving a reward at the terminal state $\termstate$.
This interaction induces a probability distribution over trajectories, denoted as \(\mathbb{P}_{\pi, s}\).
Formally, for any $h\in[H+1]$ and $s \in\cS_h$, we write $\trajectoryrand\sim\bP_{\pi,s}$ to denote a trajectory $\trajectoryrand = (s,A_h,R_h,\dots,S_{H+1},A_{H+1},R_{H+1})$
distributed according to $\bP_{\pi,s}$, where
$S_h=s$, $A_i\sim \pi(S_i)$ for $i\in[h:H+1]$, $S_{i+1}\sim P(S_i,A_i)$ for $i\in[h:H]$, and $R_i\sim \cR(S_i,A_i)$ for $i\in[h:H+1]$.
For any $a \in \cA$, we also define $\bP_{\pi,s,a}$ as a distribution over trajectories when first action $a$ is taken in state $s$, and then policy $\pi$ is followed.
For $h\in[H+1]$, we write $\bPmarg{h}_{\pi,s}$ (and $\bPmarg{h}_{\pi,s,a}$) for the marginal distribution of $(S_h, A_h)$ (i.e., the state-action pair of stage $h$) 
based on the joint distribution of $\bP_{\pi,s}$ (and $\bP_{\pi,s,a}$).

For $1 \le t \le t' \le H+1$, we use the notation  
$R_{t:t'} = \sum_{u=t}^{t'} R_u$. %
At a stage $h \in [H+1]$, the state-value and action-value functions $v^\pi$ and $q^\pi$,
for a policy $\pi \in \Pi$,
and $s_h \in \cS_h$, $a \in \cA$, 
are: 
\begin{align*}
v^\pi(s_h) \defeq \bE_{\trajectoryrand\sim\bP_{\pi,s_h}} R_{h:H+1} 
\quad \text{and} \quad 
q^\pi(s_h,a) \defeq \bE_{\trajectoryrand\sim\bP_{\pi,s_h,a}} R_{h:H+1}\, .
\end{align*}
For any function $f \in \bR^{\cS}$,
we write $f_h \in \bR^{\cS_h}$ for its restriction to stage $h$.
Similarly, for any function $g \in \bR^{\cS \times \cA}$,
we write $g_h \in \bR^{\cS_h \times \cA}$ for its restriction to stage $h$.
Let $\piopt \in \Pi$ be the optimal policy, satisfying $q^{\piopt}(s,a) \defeq \sup_{\pi\in\Pi} q^\pi(s,a)$ for all $(s,a) \in \cS \times \cA$.
The optimal policy is known to exist and is deterministic \citep{puterman2014markov}.

\subsection{Assumptions and Objective} \label{sec:assumptions and objective}

For a fixed feature map $\phi: \cS \times \cA \to \cB(\featurebound), \, \featurebound > 0$ (known to the learner), the function class at stage $h \in [H+1]$ containing linear functions is defined as follows:
\begin{align*}
    \fc_h \defeq \braces*{f: \cS_h \times \cA \to [0, H] \mid f(\cdot, \cdot) = \ip{\phi(\cdot, \cdot)}{\theta}, \theta \in \cB(\thetabound)} 
    \quad \text{for some } \thetabound > 0
    \, .
\end{align*}
The first assumption is that $\fc_h$ realizes the $q$-function of all memoryless policies exactly%
\footnote{
    We assume exact realizability for simplicity of presentation.
    Our results can be extended to the approximate case with minor modifications to the proofs, as was done by \citet{tkachuk2024trajectory}.
}.
\begin{assumption}[$q^\pi$-Realizability\footnotemark]
    \footnotetext{
        Since $\cF_h$ is a linear function class, this should be called \textbf{linear} $q^\pi$-realizability; however, since we always work with linear function classes in this paper, we omit the word ``linear'' for brevity. 
        The same is done for future assumptions that use $\cF_h$ (e.g., \cref{ass:bellman policy completeness,ass:bellman opt completeness}).
    }
    \label{ass:q-pi realizability}
    $q^\pi_h \in \fc_h$,
    for all $\pi \in \Pi$, and stages $h \in [H+1]$.
    For any $h\in[H+1]$, let $\thetatrue_h:\Pi\to\cB(\thetabound)$ be a mapping 
    that satisfies
    $q^\pi_h(\cdot, \cdot) = \ip{\phi(\cdot, \cdot)}{\thetatrue_h(\pi)}$.
\end{assumption}
Now, we introduce our assumptions on the data.
\begin{assumption}[Trajectory Data] \label{ass:trajectory data}
    The learner is given 
    $n \ge 1$ trajectories%
    \footnote{
    Our learner does not require explicit knowledge of the states within each trajectory; the features alone are sufficient.
    }
    $(\trajectory^j)_{j=1}^n$,
    where 
    each $\trajectory^j = (S_h^j, A_h^j, R_h^j)_{h \in [H+1]}$ is an independent sample from $\bP_{\pibehave, \startstate}$,
    for \textbf{behavior policy} $\pibehave \in \Pi$.
\end{assumption}
A trajectory is a sample from the joint distribution $\bP_{\pibehave, \startstate}$.
This is in contrast to just having samples from the marginal distributions $(\bPmarg{h}_{\pibehave, \startstate})_{h \in [H]}$%
\footnote{
    Since samples from the marginals are just state-action pairs,
    we of course mean that a reward and next state are also provided for each state-action pair.
    Even weaker is having samples from a general distribution over $\cS \times \cA$, for which the negative result of \citet{foster2022offline} applies.
}, 
which hide temporal dependencies between states and actions across stages that will be crucial for our learner (see \cref{sec:linear MDP modification}).
A sequence 
$\nu = (\nu_h)_{h \in [H]}$, where $\nu_h \in \dists(\cS_h \times \cA)$, is admissible 
if there exists a policy $\pi \in \Pi$ such that \begin{align*}
    \nu_h(s, a) = \bPmarg{h}_{\pi, \startstate}(s, a) \quad \text{for all } (s, a) \in \cS_h \times \cA, \ h \in [H] \, .
\end{align*}
We use $\mu \defeq (\mu_h)_{h \in [H]}$, where $\mu_h \defeq \bPmarg{h}_{\pibehave, \startstate}$, for the behavior policy's admissible distributions.
The following assumption requires the behavior policy's coverage of the state-action space to be nearly as good as any memoryless policy.
\begin{assumption}[Concentrability] \label{ass:concentrability}
    For all admissible distributions $\nu = (\nu_h)_{h \in [H]}$,
    \begin{align*}
        \textstyle\max_{h \in [H]} \max_{(s, a) \in \cS_h \times \cA} \paren{\nu_h(s, a)/\mu_h(s, a)} \le \conc,
        \quad \text{where } \conc \ge 1
        \, .
    \end{align*}
\end{assumption}
The policy optimization and evaluation objectives are as follows.
Let $\eps > 0$.
The policy optimization objective is to minimize the number of trajectories $n$, 
while outputting a policy $\hat \pi$,
such that $v^{\piopt}(\startstate) - v^{\hat \pi}(\startstate) \le \epsilon$. 
Let $\pieval$ be an arbitrary \textit{evaluation policy}.
The policy evaluation objective is to minimize the number of trajectories $n$,
while outputting a value estimate $\vout \in \bR$,
such that $\abs{v^{\pieval}(\startstate) - \vout} \le \epsilon$.

\vspace{-0.2cm}
\section{Results} \label{sec:results}

Our main result is the first polynomial sample complexity for policy evaluation
(proof:  \cref{sec:analysis of evallearner}).
\begin{theorem}[Policy Evaluation] \label{thm:policy eval}
    Let \cref{ass:q-pi realizability,ass:trajectory data,ass:concentrability} hold and the behavior policy $\pibehave$ be \textbf{known}. 
    For any evaluation policy $\pieval, \delta \in (0, 1)$ and $\eps > 0$,
    if the number of trajectories $n = \tilde \Theta(\conc^5 H^7 d^3/\eps^2)$,
    then with probability at least $1 - \delta$, 
    the value estimate $\vout$ output by \cref{alg:evallearner} satisfies
    \begin{align}
        \left| v^{\pieval}(\startstate) - \vout \right| 
        \le \eps \, .
        \label{eq:policy eval obj}
    \end{align}
\end{theorem}
The notations $\tilde \Omega, \tilde \cO$ and $\tilde \Theta$ hide polylogarithmic factors of $(1/\eps, 1/\delta, H, d, \conc, \featurebound, \thetabound)$. 
Due to our improved analysis (in \cref{proof:second claim of Qopt guarantee}), 
we obtain a sample complexity for policy optimization 
that improves the bound on $n$ by \citet{tkachuk2024trajectory} by a factor of $\conc d$
(proof: \cref{sec:analysis of optlearner}).
\begin{theorem}[Policy Optimization] \label{thm:policy opt}
    Let \cref{ass:q-pi realizability,ass:trajectory data,ass:concentrability} hold and the behavior policy $\pibehave$ be \textbf{unknown}. 
    For any $\delta \in (0, 1)$ and $\eps > 0$,
    if the number of trajectories $n = \tilde \Theta(\conc^3 H^7 d^3/\eps^2)$,
    then with probability at least $1 - \delta$, 
    the policy $\hat \pi$ output by \cref{alg:optlearner} satisfies
    \begin{align}
        v^{\piopt}(\startstate) - v^{\hat \pi}(\startstate) 
        \le \eps \, .
        \label{eq:policy opt obj}
    \end{align}
\end{theorem}

\section{Learners and Background} \label{sec:learner}

In \cref{sec:evallearner} we present our main contribution, a sample efficient learner for policy evaluation (\cref{alg:evallearner}).
We use the first three subsections to present \textbf{known} results, which we believe are crucial for understanding our contribution.
In \cref{sec:FQI/FQE} we present the fitted Q-iteration/evaluation (\fqe/\fqi) algorithms for policy evaluation and optimization respectively.
Since \fqe/\fqi works with the Bellman completeness assumptions (\cref{ass:bellman opt completeness,ass:bellman policy completeness}),
we show how to get Bellman completeness from $q^\pi$-realizability in \cref{sec:linear MDP modification}.
In \cref{sec:optlearner} we show how \citet{tkachuk2024trajectory} used these ideas along with trajectory data to develop their learner for policy optimization.
We finally present our policy evaluation learner in \cref{sec:evallearner}.

\subsection{Fitted Q-Iteration/Evaluation (FQI/FQE) with Bellman Completeness} \label{sec:FQI/FQE}
The \textit{Bellman policy operator} $T^\pi: \cup_{h \in [2:H+1]} \bR^{\cS_h \times \cA} \to \cup_{h \in [H]} \bR^{\cS_h \times \cA}$ for a policy $\pi$, 
is defined as follows: for all $h \in [H], q_{h+1} \in \bR^{\cS_{h+1} \times \cA}$,
\begin{align}
    T^\pi q_{h+1} (s, a) \defeq \bE_{\trajectoryrand \sim \bP_{\pi, s, a}} \brackets*{R_h + q_{h+1}(S_{h+1}, \pi)}, \quad \text{for all } (s, a) \in \cS_h \times \cA \, ,
    \label{eq:bellman pi operator}
\end{align}
where $q_{h+1}(s, \pi) \defeq \sum_{a \in \cA} q_{h+1}(s, a) \pi(a|s)$ is notation that will be used throughout the paper.
Notice that $q^{\pi}_h = T^\pi q^{\pi}_{h+1}$.
Since we do not have access to $T^\pi$, and $T^\pi q^{\pi}_{h+1}$ is a conditional expectation,
we can approximate it from data by solving a regression problem.
Let $\mu' = (\mu_1', \dots, \mu_H')$, be an arbitrary sequence of state-action distributions.
Only for this subsection, we assume each offline data sample $j \in [n]$ is generated as follows for all $h \in [H]$: $(S_h^j, A_h^j) \sim \mu_h'$, $R_h^j \sim \mathcal{R}(S_h^j, A_h^j)$, $\bar S_{h+1}^j \sim P(\cdot | S_h^j, A_h^j)$. 
This is more general than trajectory data, since the state-action pairs $(S_h^j, A_h^j)$ are not necessarily generated by following the behavior policy $\pibehave$, and the next state $\bar S_{h+1}^j$ need not be equal to $S_{h+1}^j$.
The empirical Bellman policy operator $\hat T^\pi: \cup_{h \in [2:H+1]} \bR^{\cS_h \times \cA} \to \cup_{h \in [H]} \bR^{\cS_h \times \cA}$ for a policy $\pi$, 
is defined as follows: for all $h \in [H], q_{h+1} \in \bR^{\cS_{h+1} \times \cA}$,
\begin{align*}
    \hat T^\pi q_{h+1}
    \defeq \textstyle\argmin_{q_h \in \fc_h} \ltfrac{1}{n} \textstyle\sum_{j=1}^{n} 
    \paren{q_h(S_h^j, A_h^j) - R_h^j - q_{h+1}(\bar S_{h+1}^j, \pi)}^2 \!\!.
\end{align*}
For policy evaluation, \fqe starts with $\hat q_{H+1} = \vec 0$ and calculates $\hat q_h = \hat T^{\pieval} \hat q_{h+1}$ for $h = H, \dots, 1$.
It then outputs $\vout = \hat q_1(\startstate, \pieval)$. 
A sufficient condition (on $\fc_h$) for \fqe to work is: 
\begin{assumption}[Bellman Completeness (for $T^\pi$)] \label{ass:bellman policy completeness}
    $T^{\pi} q_{h+1} \in \fc_h, \, \forall q_{h+1} \in \fc_{h+1}$ and $\forall h \in [H]$.
\end{assumption}
Roughly speaking, Bellman completeness ensures that the approximation error does not multiplicatively accumulate across stages.
If $\mu' = (\mu_h')_{h \in [H]}$ satisfies concentrability (\cref{ass:concentrability})
and Bellman completeness (\cref{ass:bellman policy completeness}) holds%
\footnote{
    The result holds even if each function class $\cF_h$ is not necessarily linear, but still satisfies Bellman completeness.
}, 
then the policy evaluation objective can be achieved with a polynomial number of samples that is independent of the state space size $|\cS|$ \citep{le2019batch}.
Notice how trajectory data (\cref{ass:trajectory data}) was not necessary,
if we have Bellman completeness.
Unfortunately, we cannot directly use \fqe with $q^\pi$-realizability (\cref{ass:q-pi realizability}),
since it does not imply Bellman completeness \citep{zanette2020learning}.
In fact, without further assumptions on the data, \citet{foster2022offline} showed that a policy evaluation bound (like in \cref{eq:policy eval obj}), 
cannot be achieved unless the sample complexity scales with the size of the state space $|\cS|$, which is undesirable.
Naturally, this leads to imposing an extra assumption on the data,
which we choose to be trajectory data (\cref{ass:trajectory data}),
whose benefits we discuss in the next section. 

Before moving on, we address policy optimization.
One could try to evaluate every policy using \fqe and then select the one with the largest value from the start state;
however, since the optimal policy is deterministic, there are $|\cA|^{|\cS|}$ policies to evaluate,
and thus even by taking a union bound over this set of policies we would obtain a bound that scales as
$\log(|\cA|^{|\cS|}) = |\cS| \log (|\cA|)$.
Fortunately, this can be improved if our $q$-function class $\fc$ is small or structured, 
by making use of the \textit{Bellman optimality operator}.
In particular, the Bellman optimality operator $T: \cup_{h \in [2:H+1]} \bR^{\cS_h \times \cA} \to \cup_{h \in [H]} \bR^{\cS_h \times \cA}$
and empirical Bellman optimality operator $\hat T: \cup_{h \in [2:H+1]} \bR^{\cS_h \times \cA} \to \cup_{h \in [H]} \bR^{\cS_h \times \cA}$
are defined as follows%
\footnote{
    We use \textcolor{purple}{purple} to highlight important differences between equations throughout the paper.
}
: 
for all $h \in [H]$, $\pi \in \Pi$, $q_{h+1} \in \bR^{\cS_{h+1} \times \cA}$,
\begin{align}
    T q_{h+1} (s, a) 
    &\defeq \bE_{\trajectoryrand \sim \bP_{\pi, s, a}} \brackets{R_h + \textcolor{purple}{\max\nolimits_{a'}} \, q_{h+1}(S_{h+1}, \textcolor{purple}{a'})}, \quad \text{for all } (s, a) \in \cS_h \times \cA \, .
    \label{eq:bellman opt operator}
    \\
    \hat T q_{h+1}
    &\defeq \textstyle\argmin_{q_h \in \fc_h} \frac{1}{n} \sum_{j=1}^{n} 
    \paren{q_h(S_h^j, A_h^j) - R_h^j - \textcolor{purple}{\max_{a'}} \, q_{h+1}(\bar S_{h+1}^j, \textcolor{purple}{a'})}^2 \!\! .
    \label{eq:bellman opt operator estimate}
\end{align}
For policy optimization, \fqi starts with $\hat q_{H+1} = \vec 0$ and calculates $\hat q_h = \hat T \hat q_{h+1}$ for $h = H, \dots, 1$.
It then outputs $\piout(\cdot) = \argmax_a \hat q(\cdot, a)$.
The Bellman completeness assumption for $T$ is:
\begin{assumption}[Bellman Completeness (of $T$)] \label{ass:bellman opt completeness}
    $T q_{h+1} \in \fc_h, \, \forall q_{h+1} \in \fc_{h+1}$ and $\forall h \in [H]$.
\end{assumption}
If $\mu' = (\mu_h')_{h \in [H]}$ satisfies \cref{ass:concentrability},
and Bellman completeness (\cref{ass:bellman opt completeness}) holds,
it can be shown that $\hat q$ is a good estimate of $q^{\piopt}$ in expectation.
Finally, greedifying with respect to a good estimate of $q^{\piopt}$ (see \cref{lem:policy improvement})
achieves the policy optimization objective (\cref{eq:policy opt obj}) with a polynomial sample complexity \citep{munos2008finite,chen2019information}.
Now, let us see what happens when we consider $q^\pi$-realizability, and why trajectories are helpful.

\subsection{From $q^\pi$-realizability to Bellman Completeness} \label{sec:linear MDP modification}

What we will see in this subsection is that although an MDP that is $q^\pi$-realizable is not necessarily Bellman complete,
there exists a minor modification of the MDP such that Bellman completeness does hold in the modified MDP.
Then, if we knew the modified MDP we could simply use FQE/FQI, and since the modification is minor, it will imply good bounds in the original MDP.
This observation was made by \citet{weisz2023online}%
\footnote{
    \citet{weisz2023online} studied the online RL setting with linear $q^\pi$-realizability.
    Since then \citet{pmlr-v291-mhammedi25b} and \citet{tkachuk2024trajectory} have also made use of this helpful observation.
}, 
and is necessary background for the following subsections.

To relate the $q^\pi$-realizability assumption to Bellman completeness, 
we will make use of another commonly used assumption called \textit{linear MDP} \citep{jin2020provably}.
Importantly, a linear MDP always satisfies Bellman completeness and $q^\pi$-realizability, but not vice versa \citep{zanette2020learning}.
As such, if we can show that a modification to the original $q^\pi$-realizable MDP gives a linear MDP, then we will also have Bellman completeness in the modified MDP.
We begin by building some intuition for why a $q^\pi$-realizable MDP can fail to be a linear MDP.
Informally, a linear MDP is an MDP where the transition dynamics and expected rewards can be expressed as linear functions of the same state-action features $\phi$.
It is easy to construct an MDP for which there is a state that has the same $q$-function for all policies and actions;
however, it does not have the same reward function or transition dynamics for all actions.
Such an example is given on the left in \cref{fig:skipping example}\footnote{
    A reader familiar with the work of \citet{weisz2023online} may wonder why our figure is different from their Figure 1.
    First, this view of modifying the MDP is identical from a value function perspective,
    and secondly we find this view to be more aligned with the skippy Bellman operators and skippy policies we will define later.
}
, where in state $s_1$ the $q$-function is $1$ for all policies and actions;
however, the reward for taking the up action is $1$, while the reward for taking the down action is $0.5$.
As such, to satisfy $q^\pi$-realizability, the features for the linear function class $\fc$ can be defined as $\phi(s_1, \cdot) = 1$, since $\ip{\phi(s_1, \cdot)}{1} = 1$.
But, such a feature map cannot capture the difference in rewards between the two actions, as it assigns the same value to both the up and down action.
A similar issue arises for the transitions in state $s_1$.

\begin{figure}[t!]
    \centering
    \begin{tikzpicture}[node distance=2.8cm, fatnode/.style={draw,circle, minimum size=7mm}]
        \node[fatnode] (start) {$s_1$};
        \node[fatnode] [right of=start,yshift=0.7cm] (child1) {$s_2$};
        \node[fatnode] [right of=start,yshift=-0.7cm] (child2) {$s_3$};
        \node[fatnode] [right of=child2,yshift=0.7cm] (end) {$s_4$};

        \draw[->] (start) -- (child1) node[midway,above] {\small 1};
        \draw[->,draw=red] (start) -- (child2) node[midway,below,yshift=-0.0cm] {\small 0.5};
        \draw[->] (child1) to [out=20,in=140] (end) node[midway,above,xshift=3.4cm,yshift=0.85cm] {\small 0};
        \draw[->] (child1) to [out=-30,in=170] (end) node[midway,above,xshift=3.9cm,yshift=0.2cm] {\small 0};
        \draw[->] (child2) to [out=30,in=-170] (end) node[midway,below,xshift=3.4cm, yshift=0.10cm] {\small 0.5};
        \draw[->] (child2) to [out=-10,in=-140] (end) node[midway,below,xshift=4.4cm, yshift=-0.7cm] {\small 0.5};
        \draw[->] (end) to [out=20,in=-20] (end) node[midway,below,xshift=6.2cm, yshift=0.23cm] {\small 0};

        \node[fatnode] (start2) [right of=end,xshift=-0.5cm] {$s_1$};
        \node[fatnode] [right of=start2,yshift=0.7cm] (child12) {$s_2$};
        \node[fatnode] [right of=start2,yshift=-0.7cm] (child22) {$s_3$};
        \node[fatnode] [right of=child22,yshift=0.7cm] (end2) {$s_4$};

        \draw[->] (start2) -- (child12) node[midway,above] {\small 1};
        \draw[->,draw=red] (start2) to [out=-30,in=220] (child12) node[midway,below,xshift=9.3cm, yshift=-0.3cm] {\small 1};
        \draw[->] (child12) to [out=20,in=140] (end2) node[midway,above,xshift=11.4cm,yshift=0.85cm] {\small 0};
        \draw[->] (child12) to [out=-30,in=170] (end2) node[midway,above,xshift=11.9cm,yshift=0.2cm] {\small 0};
        \draw[->] (child22) to [out=30,in=-170] (end2) node[midway,below,xshift=11.3cm, yshift=0.10cm] {\small 0.5};
        \draw[->] (child22) to [out=-10,in=-140] (end2) node[midway,below,xshift=12.3cm, yshift=-0.7cm] {\small 0.5};
        \draw[->] (end2) to [out=20,in=-20] (end2) node[midway,below,xshift=14.1cm, yshift=0.23cm] {\small 0};

    \end{tikzpicture}
    \caption{
        The features for both MDPs are $\phi(s_1, \cdot)= 1, \phi(s_3, \cdot) = 0.5, \phi(\cdot, \cdot) = 0$ otherwise.
        \textbf{Left:} $q^\pi$-realizable. 
        \textbf{Right:} Linear MDP, due to skipping $s_1$ via the up action. 
        }
    \label{fig:skipping example}
    \vspace{-0.3cm}
\end{figure}

To formally describe such problematic states we define the following function:
\begin{align*}
    \range(s) \defeq \textstyle\sup_{\pi\in\Pi} \textstyle\max_{a,a'\in\cA} q^\pi(s,a)-q^\pi(s,a'),
    \quad \text{ for all } %
    s \in \cS \, .
\end{align*}
Then, if the range is zero for a state $s$, we can encounter the issue described above.
More generally, we can encounter issues representing the rewards or transitions if there are states with low range.

One way to address this issue is to modify the MDP such that all the rewards and transitions are the same 
at any states $s$ with range less than a threshold $\alpha \ge 0$. 
In particular, at each such state, select any action $a$ and make the reward and transition of every other action equal to that of $a$.
We will refer to modifying a state in this way as \textit{skipping} the state 
via action $a$.
An example of this is shown on the right in \cref{fig:skipping example}, where we have skipped state $s_1$ via the up action.
An MDP modified in this way is indeed linear (with parameter norm bound $\thetabound/\alpha$) and thus satisfies Bellman completeness 
(see \cref{lem:qpi realizability implies bellman completeness}).
At the end of this subsection we will see how the error due to modifying the MDP depends on $\alpha$. 
But first, we formalize the above intuition that the modified MDP is linear.

As we have seen, to modify the MDP we need to know the range function.
But, the learner is not given the range function, so it uses a parametric representation of it, 
to avoid estimating the $q$-function of every policy.
For any $h \in [H]$, let $\Thetatrue_h \defeq \cl(\{\thetatrue_h(\pi)\,:\, \pi\in\Pi\}) \subseteq \cB(\thetabound)$ be a (compact) set containing the parameter values of all policies.
When $q^\pi$-realizability holds the range function can be expressed in terms of the parameters in these sets:
\begin{align*}
    \textstyle \range(s) = \textstyle\sup_{\thetatrue_h \in \Thetatrue_h} \textstyle\max_{a,a'\in\cA} \ip{\phi(s, a) - \phi(s, a')}{\thetatrue_h}, 
    \quad \text{ for all } 
    s\in\cS_h,
    h\in[H] \, . 
\end{align*}
Now we will construct a parametric bound on the range.
For all $h\in[H]$,  
fix a subset $\trueG_h \subseteq \Thetatrue_h$ of size $|\trueG_h|=d_0:=\lceil 4d\log\log(d)+16 \rceil$ that is the basis of a near-optimal design for $\Thetatrue_h$ (i.e., satisfying \cref{def:nearopt}).
The existence of such a near-optimal design follows from \citep[Part (ii) of Lemma 3.7]{todd2016minimum}.
Any $G = (G_h)_{h \in [H]} \in\Gall$,
where $\Gall \defeq (\cB(\thetabound)^{d_0})^{H}$,
can be used to define an approximate range function that is completely specified by $\tilde O(Hd^2)$ parameters:
\begin{align*}
    \textstyle \range_G(s) \defeq \max_{\vartheta \in G_h} \max_{a,a'\in\cA} \ip{\phi(s, a) - \phi(s, a')}{\vartheta}, \quad\text{for all } s\in\cS_h, h\in[H] \, .
\end{align*}
Importantly, when $\trueG \defeq (\trueG_h)_{h \in [H]}$ is used, the $\range_{\trueG}$ function upper bounds the true $\range$:%
\begin{lemma}[Prop. 4.5 \citep{weisz2023online}]\label{lem:range-G-accurate}
$\range(s) \le \sqrt{2d} \cdot \range_{\trueG}(s)$,
for all
$s\in\cS_h$, $h\in[H]$. 
\end{lemma}
Recall that we want to modify the MDP by skipping over states with range lower than a threshold $\alpha$.
The above lemma tells us that we can achieve this by skipping over states with $\range_{\trueG}(s)$ lower than $\alpha/\sqrt{2d}$.
It will be useful to define a skipping function $\omega_G$ based on $G \in \Gall$, that indicates if we skip a state or not.
Although this function can be defined as an indicator, for technical reasons%
\footnote{The smooth skipping behavior helps with covering arguments.},
we define a smoothed version, which always skips states with $\range_G(s) < \alpha/(2\sqrt{2d})$,
never skips for states with $\range_G(s) > \alpha/\sqrt{2d}$, and linearly interpolates between the two thresholds:
\begin{align}
    \omega_G(s)
    \defeq \begin{cases}
        1, & \text{if } s \neq \termstate \text{ and } \range_{G}(s) \le \alpha / (2\sqrt{2d}); \\
        2 - 2\sqrt{2d} \cdot \range_{G}(s) / \alpha, & \text{if } s \neq \termstate \text{ and } \alpha /(2\sqrt{2d}) \le \range_{G}(s) \le \alpha / \sqrt{2d}; \\
        0, & \text{otherwise.}
    \end{cases} 
    \label{eq:omega}
\end{align}
Since each $G \in \Gall$ defines a skipping function $\omega_G$, 
which in turn defines a modified MDP, we will refer to $G$ as a \textit{modification} of the MDP,
and call $\trueG$ the \textit{correct} modification. 

So far we have discussed explicitly modifying the MDP. 
Instead, we will do so implicitly through modified Bellman operators, which incorporate the skippy behavior into their choice of future $q$-values%
\footnote{
    We find this makes the learners simpler to present (see \cref{sec:optlearner,sec:evallearner}), and is also done by \citet{tkachuk2024trajectory}.
}.
For $h \in [H]$, and $1\le l \le h$,
let $\trajectory = (s_t, a_t, r_t)_{l\le t \le H+1}$ be any fixed trajectory that starts from some stage $l$. 
For $\tau \in [h+1:H+1]$ let $F_{G, \trajectory, h+1}(\tau) = (1 - \omega_G(s_\tau)) \prod_{u = h+1}^{\tau-1} \omega_G(s_u)$
be the probability of stopping at stage $\tau$
when starting from state $s_h$ and skipping subsequent states with probability $\omega_G(\cdot)$.
For any sequence of objects $a_1, \dots, a_{H+1}$,
define $a_{\futurestages} \defeq (a_{h+1}, \dots, a_{H+1})$, for all $h \in [H]$.
For a set $\cC$ define $\cC^{\cS\!\cA_{\futurestages}} \defeq \cC^{\cS_{h+1} \times \cA} \times \cdots \times \cC^{\cS_{H+1} \times \cA}$. 

For a policy $\pi$ and modification $G$, 
the \textit{skippy Bellman policy operator} and empirical version 
$T^\pi_G, \hat T^{\pi}_{G}: \cup_{h \in [2:H+1]} \bR^{\cS\!\cA_{\futurestages}} \to \cup_{h \in [H]} \bR^{\cS_h \times \cA}$ 
are defined as follows: for all $h \in [H], q_{\futurestages} \in \bR^{\cS\!\cA_{\futurestages}}$,
\begin{align}
    T_{G}^{\pi} \, q_{\futurestages} (s, a)
    &\defeq \bE_{\trajectoryrand \sim \bP_{\pibehave, s, a}} g_G^{\pi}(q_{\futurestages}, \trajectoryrand), \quad \text{for all } (s, a) \in \cS_h \times \cA, 
    \label{eq:skippy bellman policy operator} \\
    \hat T^{\pi}_{G} \, \hat q_{\futurestages}
    &\defeq \textstyle\argmin_{q_h \in \modfc_h} \ltfrac{1}{n} \textstyle\sum_{j=1}^{n} 
    \paren{q_h(S_h^j, A_h^j) - g^{\pi}_{G}(q_{\futurestages}, \trajectory^j_{\futurestages})}^2 \!\!,
    \, \text{ where}
    \label{eq:skippy bellman policy operator estimate}
    \\
    g_G^{\pi}(q_{\futurestages}, \trajectoryrand) 
    &\defeq \textstyle\sum_{\tau = h+1}^{H+1} (R_{h:\tau-1} + q_\tau(S_\tau, \pi)) \, F_{G, \trajectoryrand, h+1}(\tau)
    \nonumber
    \, , 
\end{align}
$\trajectory^j_{\futurestages} \defeq (S_t^j, A_t^j, R_t^j)_{t \in [h+1:H+1]}$, 
and $\modfc_h \supseteq \fc_h$ is defined in \cref{eq:mod function class}%
\footnote{
    The target $g^\pi_G(q_{\futurestages}, \trajectoryrand)$ is equivalent to a state-dependent $\lambda$-return \citep{sutton2018reinforcement}, where the continuation parameter $\lambda$ is given by the skipping probability $\omega_G$.
}
.
\Cref{fig:target computation} shows how the target $g_G^\pi(q_{\futurestages}, \trajectoryrand)$ is computed for a single sampled trajectory $\trajectoryrand \sim \bP_{\pi_b,s_h,a}$, and how the skipping probabilities $\omega_G(S_t)$ determine the distribution over the stopping time $\tau$.
\begin{figure}[t]
    \centering
    \begin{tikzpicture}[
    >=Latex,
    font=\small,
    state/.style={draw, rounded corners, minimum width=1.9cm, minimum height=0.95cm, align=center},
    stopbox/.style={draw, rounded corners, align=left, text width=5.2cm, inner sep=6pt},
    lab/.style={align=center}
    ]

    \node[state] (sa) {$(s_h,a)$\\fixed};
    \node[state, right=2.2cm of sa] (s1) {$S_{h+1}$};
    \node[state, right=2.2cm of s1] (s2) {$S_{h+2}$};
    \node[state, right=2.2cm of s2] (s3) {$S_{h+3}$};

    \draw[->, thick] (sa) -- node[above=2pt] {$R_h$} (s1);
    \draw[->, thick] (s1) -- node[above=2pt, align=center] {$A_{h+1}\sim\pi^b$\\$R_{h+1}$} (s2);
    \draw[->, thick] (s2) -- node[above=2pt, align=center] {$A_{h+2}\sim\pi^b$\\$R_{h+2}$} (s3);

    \node[lab, above=0.8cm of s1] (w1) {$\omega_G(S_{h+1})=1$\\always skip};
    \node[lab, above=0.8cm of s2] (w2) {$\omega_G(S_{h+2})=0.7$\\skip w.p.\ $0.7$};
    \node[lab, above=0.8cm of s3] (w3) {$\omega_G(S_{h+3})=0$\\never skip};

    \draw[densely dashed] ($(s1.north)+(0,0.03)$) -- ($(w1.south)+(0,-0.03)$);
    \draw[densely dashed] ($(s2.north)+(0,0.03)$) -- ($(w2.south)+(0,-0.03)$);
    \draw[densely dashed] ($(s3.north)+(0,0.03)$) -- ($(w3.south)+(0,-0.03)$);

    \end{tikzpicture}
    \caption{
        An example of computing the target $g_G^{\pi}(q_{\futurestages}, \trajectoryrand)$ for a trajectory $\trajectoryrand = (s_h, a, R_h, S_{h+1}, A_{h+1}, \dots) \sim \bP_{\pi^b,s_h,a}$.
        Starting from $(s_h,a)$, the target is computed by following the trajectory and accumulating rewards while skipping continues,
        using the value estimate $q_\tau(S_\tau,\pi)$ at the first stage $\tau$ where skipping stops, 
        and then averaging over the random stopping time $\tau$. 
        For the displayed trajectory, 
        $\omega_G(S_{h+1})=1$, $\omega_G(S_{h+2})=0.7$, and $\omega_G(S_{h+3})=0$, so
        skipping stops at stage $h+1$ with probability $F_{G,\mathrm{Traj},h+1}(h+1)=(1-\omega_G(S_{h+1}))=0$,
        stops at stage $h+2$ with probability
        $F_{G,\mathrm{Traj},h+1}(h+2)=(1-\omega_G(S_{h+2}))\omega_G(S_{h+1})=0.3$
        and stops at stage $h+3$ with probability
        $F_{G,\mathrm{Traj},h+1}(h+3)=(1-\omega_G(S_{h+3}))\omega_G(S_{h+1})\omega_G(S_{h+2})=0.7$.
        The target is 
        $g_G^{\pi}(q_{\futurestages}, \trajectoryrand)
        =0.3\bigl(R_h+R_{h+1}+q_{h+2}(S_{h+2},\pi)\bigr)
        +0.7\bigl(R_h+R_{h+1}+R_{h+2}+q_{h+3}(S_{h+3},\pi)\bigr)$.
        The values of $\omega_G(S_t)$ were chosen for illustration purposes.
    }
    \label{fig:target computation}
\end{figure}
If the modification $G$ is such that no states are skipped (i.e., $\tau = h+1$ for all states), then $g^\pi_G(q_{\futurestages}, \trajectoryrand) = R_h + q_{h+1}(S_{h+1}, \pi)$, 
which reduces \cref{eq:skippy bellman policy operator} to the usual Bellman policy operator (\cref{eq:bellman pi operator}).
Similarly, 
for a modification $G$,
the \textit{skippy Bellman optimality operator} and empirical version 
$T_G, \hat T_G: \cup_{h \in [2:H+1]} \bR^{\cS\!\cA_{\futurestages}} \to \cup_{h \in [H]} \bR^{\cS_h \times \cA}$ 
are defined as follows: for all $h \in [H]$, $q_{\futurestages} \in \bR^{\cS\!\cA_{\futurestages}}$,
\begin{align}
    T_{G} \, q_{\futurestages} (s, a)
    &\defeq \bE_{\trajectoryrand \sim \bP_{\pibehave, s, a}} \textcolor{purple}{g_G}(q_{\futurestages}, \trajectoryrand), 
    \quad \text{for all } (s, a) \in \cS_h \times \cA, 
    \nonumber
    \\
    \hat T_{G} \, \hat q_{\futurestages}
    &\defeq \textstyle\argmin_{q_h \in \modfc_h} \ltfrac{1}{n} \textstyle\sum_{j=1}^{n} 
    \paren{q_h(S_h^j, A_h^j) - \textcolor{purple}{g_{G}}(q_{\futurestages}, \trajectory^j_{\futurestages})}^2 \!\!,
    \, \text{ where}
    \nonumber
    \\
    \textcolor{purple}{g_G}(q_{\futurestages}, \trajectoryrand) 
    &\defeq \textstyle\sum_{\tau = h+1}^{H+1} (R_{h:\tau-1} + \textcolor{purple}{\max_{a'}} \, q_\tau(S_\tau, \textcolor{purple}{a'})) \, F_{G, \trajectoryrand, h+1}(\tau)
    \label{eq:skippy bellman opt target}
    \, . 
\end{align}

\begin{remark} \label{rem:trajectories important}
    Notice that the targets
    $g^{\pi}_{G}(q_{\futurestages}, \trajectory^j_{\futurestages})$ 
    and
    $g_{G}(q_{\futurestages}, \trajectory^j_{\futurestages})$ 
    are nonlinear functions of the trajectory $\trajectory^j_{\futurestages}$, 
    due to the product form of $F_{G, \trajectory^j, h+1}(\tau)$.
    This is why trajectory data (i.e., samples from the distribution $\bP_{\pibehave, \startstate}$) are important, and samples from the marginal distributions $\bP_{\pibehave, \startstate}^h$ are likely not sufficient.
\end{remark}

The reason we will care about these skippy Bellman operators is that they will satisfy two properties that we needed for \fqe/\fqi to work in \cref{sec:FQI/FQE}.
First, by a slight modification to Lemma 4.2 in \citet{tkachuk2024trajectory}, if we have $q^\pi$-realizability (\cref{ass:q-pi realizability}),
then Bellman completeness holds (with a larger function class) for the skippy Bellman operators based on the correct modification $\trueG$ (proof: \cref{proof:qpi realizability implies bellman completeness}).
\begin{lemma}[$q^\pi$-realizability $\implies$ skippy Bellman completeness] \label{lem:qpi realizability implies bellman completeness}
    Assume $(\fc_h)_{h \in [H]}$ satisfies $q^\pi$-realizability (\cref{ass:q-pi realizability}).
    Define a larger function class $\modfc_h \supseteq \fc_h$ as:
    \begin{align}
        \modfc_h \defeq \braces{f: \cS_h \times \cA \to \bR \mid f(\cdot, \cdot) = \ip{\phi(\cdot, \cdot)}{\theta}, \theta \in \cB(\modthetabound)} \, ,
        \text{ for all } h \in [H+1],
        \label{eq:mod function class}
    \end{align}
    where $\modthetabound := \thetabound \cdot (8H^2d_0 \sqrt{2d} /\alpha+1)$. 
    Then,
    for any $h \in [H]$, $q_{\futurestages} \in [0,H]^{\cS\!\cA_{\futurestages}}$ with $q_{H+1}(\termstate, \cdot) = 0$,
    \begin{align*}
        T_{\trueG} \, q_{\futurestages} \in \modfc_h \quad \text{and} \quad T_{\trueG}^{\pieval} q_{\futurestages} \in \modfc_h \, .
    \end{align*}
\end{lemma}
Notice that we have abused the Bellman completeness naming convention, since the above holds for arbitrary $q_{h+1}, \dots, q_{H+1}$, that are not necessarily in $\modfc_{h+1}, \dots, \modfc_{H+1}$ respectively.
This is a stronger property than in the regular Bellman completeness assumptions (\cref{ass:bellman opt completeness,ass:bellman policy completeness}), and will be important for the proofs of \cref{lem:Qopt guarantee,lem:Qeval guarantee}.

Recall that with the regular Bellman policy operator $q^\pi_h = T^\pi q^\pi_{h+1}$.
The second useful property is that a similar relationship holds for the skippy Bellman policy operators, except for skippy policies.
In particular, for a policy $\pi$ define its skippy version based on a modification $G$ as the policy that at each state $s$ 
chooses an action according to $\pibehave(a|s)$ (this is the skippy action) with probability $\omega_G(s)$ and according to $\pi(a|s)$ with probability $1 - \omega_G(s)$%
\footnote{
    There is a relationship between a skippy policy $\pi_G$ and a modified MDP based on $G$.
    In particular, the $q$-function for the skippy policy $q^{\pi_G}$ in the original MDP is equal to the $q$-function for $\pi$, $q^{\pi}$, in the modified MDP.
}:
\begin{align}
    \pi_G(a|s) 
    &\defeq \pibehave(a|s) \omega_G(s) + \pi(a|s) (1 - \omega_G(s)) \, .
    \label{eq:pievalG definition}
\end{align}
The following useful property holds for the skippy Bellman policy operator (proof: \cref{proof:skippy bellman policy operator backup}):
\begin{lemma} \label{lem:skippy bellman policy operator backup}
    $q^{\pi_G}_h = T_G^{\pi} \, q_{\futurestages}^{\pi_G}$,
    for any policy $\pi \in \Pi$ and modification $G \in \Gall$.
\end{lemma}
This reduces to $q^\pi_h = T^\pi q^\pi_{h+1}$,
if $G$ is such that no states are skipped. 
If we knew $\trueG$, 
then for policy evaluation,
this motivates running \fqe with the empirical skippy Bellman policy operator $\hat T^{\pieval}_{\trueG}$.

For policy optimization, we do not have access to $\piopt$, and thus cannot estimate $q^{\piopt}$ using \cref{eq:skippy bellman policy operator estimate},
since $g^{\piopt}_{\trueG}$ depends on $\piopt$.
Instead, we define a skippy optimal policy $\pioptg_G$ that does not depend on $\piopt$: 
which at state $s$, 
with probability $1 - \omega_G(s)$,
takes the greedy action w.r.t. its future $q$-values,
\begin{align*}
    \pioptg_G(a|s) 
    &\defeq \pibehave(a|s) \omega_G(s) + \Ismall{a = \textstyle\argmax_{a' \in \cA} q^{\pioptg_G}_h(s, a')} (1 - \omega_G(s)) \, .
\end{align*}
A similar result to $q^{\piopt}_h = T q^{\piopt}_{h+1}$ for the Bellman optimality operator holds for the skippy Bellman optimality operator with the skippy optimal policy (proof: \cref{proof:skippy bellman optimality operator backup}):
\begin{lemma} \label{lem:skippy bellman optimality operator backup}
    $q^{\pioptg_G}_h = T_G \, q_{\futurestages}^{\pioptg_G}$, for all modifications $G \in \Gall$.
\end{lemma}
This motivates running \fqi with the empirical skippy Bellman optimality operator $\hat T_{\trueG}$.

Finally, the error due to modifying a policy with $\trueG$ is controlled by $\alpha$
(proof: \cref{proof:skippy policy error}):
\begin{lemma} \label{lem:skippy policy error}
    $\abs{v^{\pi} - v^{\pi_{\trueG}}} \le H \alpha$, for any $\pi \in \Pi$,
    and $v^{\piopt} - v^{\pioptg_{\trueG}} \le H \alpha$.
\end{lemma}

This result makes intuitive sense: if we skip less states (i.e., make $\alpha$ smaller), then our modification error decreases. 
But, in \cref{lem:qpi realizability implies bellman completeness} the size of the function class $\modfc$ increases as we make $\alpha$ smaller.
This shows how $\alpha$ trades off between modification error and function class size.

Unfortunately, the above approaches for policy evaluation and optimization require knowledge of $\trueG$,
which we do not have. 
We are not aware of a sample efficient way to learn the correct modification $\trueG$, 
since learning the $\range$ function seems at least as hard as solving the original problem.
Fortunately, in the next two sections we will see that we do not need to learn $\trueG$ directly.

\subsection{A Learner for Policy Optimization with $q^\pi$-realizability} \label{sec:optlearner}
In this subsection we present the policy optimization learner (\cref{alg:optlearner}) by \citet{tkachuk2024trajectory},
which is sample efficient when $q^\pi$-realizability holds (see \cref{thm:policy opt}).

Recall that if we knew $\trueG$ we would simply run FQI, which estimates $q^{\pioptg_{\trueG}}$ well, and then greedifies.
Since we do not know $\trueG$, we construct a set of estimates,
\begin{align*}
    \Qopt
    = \braces{
        q^G \in \bR^{\cS \times \cA} : q^G_h = \hat T_G \, q^G_{\futurestages}, \, \text{ for } h \in [H], G \in \Gopt, \, q^G_{H+1} = \vec 0
    } \, ,
\end{align*}
that contains at least one good estimate of $q^{\pioptg_{\trueG}}$,
and every element of $\Qopt$ is a good estimate of some $q^{\pioptg_G}, G \in \Gopt$.
The set $\Gopt \subseteq \Gall$ is defined in \cref{sec:Qopt and Qeval}, 
and uses the existence of $\trueG$ to ensure the preceding claims hold.
The next lemma formalizes the above claims 
(proof: \cref{proof:Qopt guarantee}).
\begin{lemma}[$\Qopt$ guarantee] \label{lem:Qopt guarantee}
    For all $\delta \in (0,1)$, 
    with probability at least $1 - \delta$, 
    \begin{enumerate}
        \vspace{-0.1cm}
        \item 
            for all $q^G \in \Qopt$,
            admissible distributions $\nu = (\nu_t)_{t \in [H]}$, 
            and stages $h \in [H]$,
            \vspace{-0.2cm}
            \begin{align*}
                &\bE_{(S, A) \sim \nu_h} \abs{q^G(S, A) - q^{\pioptg_G}(S, A)} 
                \le \tilde \eps 
                = \bigOt\paren{\log(1/\alpha) \conc^{3/2} H^{5/2} d^{3/2} n^{-1/2}},
                \,\,\, \text{defn \cref{eq:tilde eps defn}} \, ,
            \end{align*} 
        \vspace{-0.8cm}
        \item and there exists a $q^G \in \Qopt$ such that $G = \trueG$.
    \end{enumerate}
\end{lemma}
Although we use the $\bigOt$ notation to hide logarithmic factors,
we explicitly write the $\log(1/\alpha)$ factor, since it will be important to keep track of it when we set $\alpha$ at the end.
For exposition only, we make two simplifying assumptions%
\footnote{
    The simplifying assumptions are not needed for the actual analysis, and are only made for presentation clarity.
}: 
(i) in the first claim of the lemma we have perfect estimates (i.e., $\tilde \eps = 0$),
and 
(ii) that the bound holds pointwise instead of in expectation.
Together, these simplifications imply the following very useful properties: 
    $\Qopt$ \textbf{only} contains the $q$-functions of $\pioptg_G$ under different modifications (i.e., $\Qopt \subseteq \braces{q^{\pioptg_G} : G \in \Gall}$), 
and $q^{\pioptg_{\trueG}} \in \Qopt$.

Based on these properties, this suggests a natural strategy for the learner.
Recall that $q^{\piopt} \ge q^\pi$ for any policy $\pi$, 
and that $q^{\piopt} - q^{\pioptg_{\trueG}} \le H \alpha$ (\cref{lem:skippy policy error}).
Thus, the $q$-function in $\Qopt$ that has the highest value at the start state 
must estimate $q^{\piopt}(\startstate)$ with error at most $H \alpha$.
We have also seen at the end of \cref{sec:FQI/FQE} that greedifying w.r.t. a good estimate of $q^{\piopt}$ gives a near optimal policy.
Thus, the learner should simply select the $q$-function in $\Qopt$ that has the highest value at the start state,
and greedify w.r.t. it.
Fortunately, this same idea extends to the case without the simplifying assumptions (see \cref{sec:analysis of optlearner}).
The final error is $H\alpha + 2H \tilde \eps$ (see \cref{eq:opt final bound}),
and $\alpha$ is set to balance the two terms.
This learner is shown in \cref{alg:optlearner}, and given the name \optlearner since the construction of $\Qopt$ is related to \fqi.
The analysis in \cref{sec:analysis of optlearner} formalizes the above steps. 

\begin{algorithm}[h]
\small
  \caption{\optlearner}\label{alg:optlearner}
  \begin{algorithmic}[1]
    \Require accuracy $\eps$, failure probability $\delta$, concentrability coefficient $\conc$, 
        trajectories $(\trajectory^j)_{j \in [n]}$. 
        \vspace{0pt}
    \State $\hat q \gets \argmax_{q \in \Qopt} \max_a q(\startstate, a)$ 
        \vspace{0pt}
    \State \Return $\piout(s) \gets \argmax_a \hat q(s, a)$
  \end{algorithmic}
\end{algorithm}

\optlearner is the learner defined by \citet{tkachuk2024trajectory}. 
The only difference is in the presentation, since \citet{tkachuk2024trajectory} define $\piout$ 
as the solution to an optimization problem, where their feasible set is closely related to $\Qopt$. 
We have explicitly defined the set $\Qopt$ since it will be helpful 
for understanding the learner we present for policy evaluation in the next section.

We are not aware of a computationally efficient way to implement \optlearner.
One of the difficulties is that the feasible set $\Qopt$ is not convex, since it is defined based on $\hat T_G$,
which depends on a product of the non-convex function $\omega_G$ (see \cref{eq:skippy bellman opt target} and the definition of $F_{G, \trajectoryrand, h+1}$).

\subsection{Our Learner for Policy Evaluation with $q^\pi$-realizability} \label{sec:evallearner}

We are now ready to present our main contribution: \cref{alg:evallearner}, 
which is a sample efficient learner for policy evaluation when $q^\pi$-realizability holds (see \cref{thm:policy eval}).
Similar to the previous section, since we do not know $\trueG$, we construct a set,
\begin{align}
    \Qeval
    = \braces{
        q^G \in \bR^{\cS \times \cA} : q^G_h = \hat T^{\pieval}_G \, q^G_{\futurestages}, \, \text{ for } h \in [H], G \in \Geval, \, q^G_{H+1} = \vec 0
    } \, ,
    \label{eq:Qeval definition}
\end{align}
that contains at least one good estimate of $q^{\pieval_{\trueG}}$,
and every element of $\Qeval$ is a good estimate of some $q^{\pieval_G}, G \in \Geval$.
We defer the definition of $\Geval$ to \cref{sec:Qopt and Qeval}.
The claims are formally stated in the following lemma (which parallels \cref{lem:Qopt guarantee}, 
proof: \cref{proof:Qeval guarantee}).
\begin{lemma}[$\Qeval$ guarantee] \label{lem:Qeval guarantee}
    For all $\delta \in (0,1)$, 
    with probability at least $1 - \ltfrac{\delta}{2}$, 
    \begin{enumerate}
        \vspace{-0.1cm}
        \item 
            for all $q^G \in \Qeval$,
            admissible distributions $\nu = (\nu_t)_{t \in [H]}$, 
            and stages $h \in [H]$,
            \vspace{-0.2cm}
            \begin{align*}
                &\bE_{(S, A) \sim \nu_h} \abs{q^G(S, A) - q^{\pieval_G}(S, A)} 
                \le \tilde \eps \, ,
            \end{align*}
        \vspace{-0.8cm}
        \item and there exists a $q^G \in \Qeval$ such that $G = \trueG$.
    \end{enumerate}
\end{lemma}
Similar to the previous section, to streamline the explanation, 
we make the following simplifying assumptions%
\footnote{
    As before, the simplifying assumptions are not needed for the actual analysis, and just for presentation.
}:
(i) we have perfect estimates (i.e., $\tilde \eps = 0$),
and 
(ii) that the bound holds pointwise instead of in expectation.
These, in turn, imply the following useful properties: 
$\Qeval$ \textbf{only} contains the $q$-functions of $\pieval_G$ under different modifications (i.e., $\Qeval \subseteq \braces{q^{\pieval_G} : G \in \Gall}$), 
and $q^{\pieval_{\trueG}} \in \Qeval$.
In other words, we have access to a set $\Qeval$, which we know contains a good estimate of $q^{\pieval}$ (namely $q^{\pieval_{\trueG}}$); 
however, it can also contain $q$-functions of other policies.
The question is how to select a $q$-function from $\Qeval$ that is a good estimate of $q^{\pieval}$ from the start state?

The main contribution of our work is providing a sample efficient answer to the above question. 
To build some intuition for how we can do this let us consider the difference between the policy optimization and evaluation objectives.
In policy evaluation we know the policy $\pieval$, which we are tasked with evaluating;
however, we do not know the $q$-value of $\pieval$. 
On the other hand, in policy optimization we do not know the optimal policy $\piopt$;
however, we do know that the value function of $\piopt$ is the largest possible (i.e., $q^{\piopt} \ge q^\pi$ for any policy $\pi$).
Importantly, we made use of the latter to design \optlearner 
when we selected the $q$-function in $\Qopt$ that had the largest value at the start state.
Unfortunately, the same strategy does not work for policy evaluation since we do not know the value of $\pieval$.
Instead, we make use of our knowledge of $\pieval$.

If our learner 
outputs the value estimate $q^{\pieval_G}(\startstate, \pieval_G) = v^{\pieval_G}(\startstate)$ for some $q^{\pieval_G} \in \Qeval$,
then, by the performance difference lemma (\cref{lem:performance difference lemma}), 
\begin{align}
    \abs{v^{\pieval}(\startstate) - v^{\pieval_G}(\startstate)}
    &= \abs{\textstyle\sum_{h=1}^H \bE_{(S_h, A_h) \sim \bPmarg{h}_{\pieval, s_1}} \paren{q^{\pieval_G}(S_h, \pieval) - q^{\pieval_G}(S_h, \pieval_G)}}
    \label{eq:pdl evallearner} \, .
\end{align}
The quantity inside the expectation is often called the \textit{advantage},
and a bound on the expected advantage for all $h$, clearly
implies a bound on the error of our learner's value estimate.
Importantly, the expected advantage 
only depends on $q^{\pieval_G}$, $\pieval$, and $\pieval_G$, all of which are known,
and does not depend on $q^{\pieval}$, which we are trying to estimate and is thus unknown.
\begin{remark}
    Knowing $\pieval_G$ requires knowing $\pibehave$ and $G$.
    This is why we assume that we know $\pibehave$ in \cref{thm:policy eval} and \cref{alg:evallearner}.
    Removing this assumption is an interesting direction for future work, and has been studied in related settings \citep{zhan2022offline,ozdaglar2023revisiting}.
    Furthermore, 
    to have access to $G$ we should have defined the set $\Qeval$ to contain pairs $(q^G, G)$;
    however, to simplify notation we have omitted $G$ and assume we know it implicitly for each $q^G \in \Qeval$.
\end{remark}
Notice that when $G = \trueG$,
$\max_{s \in \cS} \left| q^{\pieval_{\trueG}}(s, \pieval) - q^{\pieval_{\trueG}}(s, \pieval_{\trueG}) \right| \le \alpha$, 
since $\pieval$ and $\pieval_{\trueG}$ only differ on states $s$ with $\range(s) \le \alpha$.
Under the simplifying assumptions (where we have pointwise access to $q^{\pieval_G}$), 
the learner should select the $q^{\pieval_G} \in \Qeval$ that minimizes
$\max_{s \in \cS} \left| q^{\pieval_G}(s, \pieval) - q^{\pieval_G}(s, \pieval_G) \right|$,
since the value estimate $v^{\pieval_G}(\startstate)$ will be at most $H \alpha$ away from $v^{\pieval}(\startstate)$ (by \cref{eq:pdl evallearner}).

Unfortunately, this idea does not quite work when we only have approximately good estimates of $q^{\pieval_G}$ in expectation (as in \cref{lem:Qeval guarantee}).
The reason for this is that the learner would now need to select the $q^G \in \Qeval$ that minimizes
$\max_{s \in \cS} \left| q^G(s, \pieval) - q^G(s, \pieval_G) \right|$.
With the simplifying assumptions, to show this quantity would not be too large we relied on the fact that 
$\max_{s \in \cS} \left| q^{\pieval_{\trueG}}(s, \pieval) - q^{\pieval_{\trueG}}(s, \pieval_{\trueG}) \right| \le \alpha$.
To make use of this again we would need a pointwise bound on $|q^G(s, a) - q^{\pieval_{\trueG}}(s, a)|$ for some $q^G \in \Qeval$. 
But, by the second claim in \cref{lem:Qeval guarantee} we only have a bound in expectation.
Fortunately, this will be enough, since if we look back at \cref{eq:pdl evallearner} we see that we only need to bound
the \textbf{expected} advantage.

As such, the learner (that does indeed work without any simplifying assumptions) 
minimizes the following estimate of the expected advantage over all $q^G \in \Qeval$
\begin{align}
    \textstyle\max_{h \in [H]} \ltfrac{1}{n} \sum_{j=1}^{n} \abs{q^G(S_h^j, \pieval) - q^G(S_h^j, \pieval_G)} \, ,
    \label{eq:advantage estimate}
\end{align}
and then outputs the value estimate $\vout = q^{\outG}(\startstate, \pieval_{\outG})$,
where $\outG \in \Geval$ is such that $q^{\outG}$ is the minimizer of \cref{eq:advantage estimate} over $\Qeval$.
The final error is $\bigOt\paren{H\alpha + \conc H \tilde \eps}$ (see \cref{eq:eval final bound}),
and $\alpha$ is set to balance the two terms.
We call this learner \evallearner (\cref{alg:evallearner}), since the construction of $\Qeval$ (see \cref{sec:Qopt and Qeval}) is based on \fqe%
\footnote{
    Similar to \optlearner, we are not aware of a computationally efficient way to implement \evallearner since the feasible set $\Qeval$ is not convex for the same reasons as $\Qopt$.
}
.
The analysis in \cref{sec:analysis of evallearner} formalizes the above steps.

\begin{algorithm}[H]
\small
  \caption{\evallearner}\label{alg:evallearner}
  \begin{algorithmic}[1]
    \Require accuracy $\eps$, fail prob $\delta$, concentrability $\conc$, 
        trajectories $(\trajectory^j)_{j \in [n]}$, 
        behavior policy $\pibehave$,
        eval policy $\pieval$
        \vspace{0pt}
    \State $q^{\outG} \gets \argmin_{q^G \in \Qeval} \max_{h \in [H]} \frac{1}{n} \sum_{j=1}^{n} \abs{q^G(S_h^j, \pieval) - q^G(S_h^j, \pieval_G)}$ 
        \vspace{0pt}
    \State \Return $\vout \gets q^{\outG}(\startstate, \pieval_{\outG})$
  \end{algorithmic}
\end{algorithm}

\subsection{Defining $\Gopt$ and $\Geval$} \label{sec:Qopt and Qeval}

In this section, we define the sets $\Gopt$ and $\Geval$ which are used to define $\Qopt$ and $\Qeval$ in \cref{sec:optlearner,sec:evallearner} respectively. 
The set $\Gopt$ is the same as the feasible set in the optimization problem defined by \citet{tkachuk2024trajectory}.
Since the modified function class $(\modfc_h)_{h \in [H]}$ contains bounded linear functions, it will be useful to define the following notation.
For $x \in \bR$, let $\clip_{[0, H]} x \defeq \max\braces{0, \min\braces{H, x}}$.
For $h \in [H], \theta \in \bR^d, \theta_{\futurestages} \in (\bR^d)^{[h+1:H+1]}$, let
\begin{align*}
    q_{\theta}(\cdot, \cdot) 
        \defeq \ip{\phi(\cdot, \cdot)}{\theta},
    &\quad q_{\theta_{\futurestages}} 
        \defeq (q_{\theta_{h+1}}, \dots, q_{\theta_{H+1}}),
    \\
    \bar q_{\theta}(\cdot, \cdot)
        \defeq \clip_{[0, H]} q_{\theta}(\cdot, \cdot),
    &\quad \bar q_{\theta_{\futurestages}} 
        \defeq (\bar q_{\theta_{h+1}}, \dots, \bar q_{\theta_{H+1}}).
\end{align*}
Let $\hat X_h \defeq \lambda I + \textstyle\sum_{j \in [n]} \phi(S^j_h, A^j_h) \phi(S^j_h, A^j_h)^\top$, with $\sqrt{\lambda} = H^{3/2}d/\modthetabound$ (defn \cref{eq:lambda defn}), 
be the unnormalized empirical covariance matrix at stage $h \in [H]$.
Now we define sets based on least squares estimates, that will be used to define $\Gopt$:
\begin{align*}
    &\thetaestopt_{G, h} 
    \defeq \braces{ 
        \textstyle\argmin_{\hat \theta \in \cB(\modthetabound)} \ltfrac{1}{n} \sum_{j=1}^{n} \paren{q_{\hat \theta}(S_h^j, A_h^j) - g_{G}(\bar q_{\theta_{\futurestages}}, \trajectory^j_{\futurestages})}^2
        : \theta_{\futurestages} \in \bigtimes_{u=h+1}^{H+1}\thetaopt_{G, u}} \, , 
    \nonumber \\ 
    &\thetaopt_{G, h} 
    \defeq \braces{\theta_h \in \cB(\modthetabound): \textstyle\min_{\hat \theta_h \in \thetaestopt_{G, h}} \snorm{\theta_h - \hat \theta_h}_{\hat X_h} \le \beta}, 
    \,\, \thetaopt_{G, H+1} = \{\vec 0\} \, ,
\end{align*}
where $\beta = \bigOt\paren{\log(1/\alpha)H^{3/2} d}$ (defn \cref{eq:beta defn}).
The sets $\thetaopt_{G, h}$ should be thought of as expanding the least squares estimates in $\thetaestopt_{G, h}$ to include parameters within a $\beta$ neighborhood.
The next lemma shows that the effect of this is that the parameter $\thetatrue_h(\pioptg_G)$ that realizes the $q$-function of policy $\pioptg_G$ 
will be in the set $\thetaopt_{G, h}$ for all modifications $G \in \Gall$.
\begin{lemma}[Lemma D.2 in \citep{tkachuk2024trajectory}] \label{lem:pievalG in Theta}
    For any $\delta \in (0,1)$,
    with probability at least $1- \delta/5$, for all $G \in \Gall$, and $h \in [H+1]$, it holds that $\thetatrue_h(\pioptg_G) \in \thetaopt_{G, h}$.
\end{lemma}
Due to the above lemma and skippy Bellman completeness (\cref{lem:qpi realizability implies bellman completeness}) holding for $\trueG$,
the $q$-function based on any parameter in $\thetaopt_{\trueG, h}$ will be a good estimate of $q_{\thetatrue_h(\pioptg_{\trueG})} = q^{\pioptg_{\trueG}}_h$.
This implies that we can construct a non-empty set based on a data dependent check of this condition: 
\begin{align}
    \Gopt 
    \defeq \braces{G \in \Gall : \ltfrac{1}{n} \textstyle\sum_{j=1}^{n} \paren{\max_{\theta \in \thetaopt_{G, h}} \bar q_{\theta}(S_h^j, A_h^j) 
    - \min_{\theta \in \thetaopt_{G, h}} \bar q_{\theta}(S_h^j, A_h^j)} \le \bar \eps, \, \forall h \in [H]} 
    \, ,
    \nonumber
\end{align}
where $\bar \eps = \bigOt\paren{\log(1/\alpha) \conc^{1/2} H^{5/2} d^{3/2} n^{-1/2}}$ (defn \cref{eq:bar eps defn}).

Next, we define the set $\Geval$.
In a similar manner to $\thetaestopt_{G, h}$ and $\thetaopt_{G, h}$, we define:
\begin{align*}
    &\textcolor{purple}{\thetaesteval_{G, h}} 
    \defeq \braces{
        \textstyle\argmin_{\hat \theta \in \cB(\modthetabound)} \frac{1}{n} \sum_{j=1}^{n} \paren{q_{\hat \theta}(S_h^j, A_h^j) - \textcolor{purple}{g^{\pieval}_{G}}(\bar q_{\theta_{\futurestages}}, \trajectory^j_{\futurestages})}^2
        : \theta_{\futurestages} \in \bigtimes_{u=h+1}^{H+1}\textcolor{purple}{\thetaeval_{G, u}}} \, , 
    \nonumber \\ 
    &\textcolor{purple}{\thetaeval_{G, h}}
    \defeq \braces{\theta_h \in \cB(\modthetabound): \textstyle\min_{\hat \theta_h \in \textcolor{purple}{\thetaesteval_{G, h}}} \snorm{\theta_h - \hat \theta_h}_{\hat X_h} \le \beta}, 
    \,\, \textcolor{purple}{\thetaeval_{G, H+1}} = \{\vec 0\} \, .
\end{align*}
The only difference between $\thetaestopt_{G, h}$ and $\textcolor{purple}{\thetaesteval_{G, h}}$ is that the latter uses the target $\textcolor{purple}{g^{\pieval}_{G}}$ instead of $g_{G}$.
The set $\Geval$ is defined similarly to $\Gopt$, but instead checks if the $q$-values based on $\textcolor{purple}{\thetaeval_{G, h}}$ are close:
\begin{align}
    \textcolor{purple}{\Geval} \defeq \braces{G \in \Gall : \textstyle \ltfrac{1}{n} \sum_{j=1}^{n} \paren{\max_{\theta \in \textcolor{purple}{\thetaeval_{G, h}}} \bar q_{\theta}(S_h^j, A_h^j) 
    - \min_{\theta \in \textcolor{purple}{\thetaeval_{G, h}}} \bar q_{\theta}(S_h^j, A_h^j)} \le \bar \eps, \, \forall h \in [H]} 
    \nonumber
\end{align}

\acks{
  We thank Alex Ayoub and Kushagra Chandak for many helpful discussions.
  V.T. acknowledges the support of the Natural Sciences and Engineering Research Council of Canada (NSERC), [PGS D - 600128 - 2025].
  X.T. acknowledges funding from the Alberta Major Innovation Fund (Amii) and NSERC Discovery Grant [RGPIN-2022-03646].
  C.S. acknowledges funding from the Canada CIFAR AI Chairs Program, Amii and NSERC.
}

\bibliography{references}

\newpage
\appendix

In \cref{sec:analysis of evallearner} we will prove our main result (\cref{thm:policy eval}).
For completeness we also prove \cref{thm:policy opt} in \cref{sec:analysis of optlearner}, which follows the exact same steps as \citep{tkachuk2024trajectory}.
The missing proofs and definitions from \cref{sec:learner} are given in \cref{sec:learner missing proofs and defs}.

\section{Analysis of \evallearner} \label{sec:analysis of evallearner}

Before proceeding we introduce a useful lemma that makes use of the concentrability coefficient (\cref{ass:concentrability}) and will be used several times in this section.
\begin{lemma} \label{lem:change of measure}
    Let $\nu = (\nu_h)_{h \in [H]}$ be an admissible distribution and let $f: \cS \times \cA \to \bR$ be a non-negative function. 
    Then, for all $h \in [H]$, $\bE_{(S, A) \sim \nu_h} f(S, A) \le \conc \cdot \bE_{(S, A) \sim \mu_h} f(S, A)$.
\end{lemma}
\begin{proof}
    $\bE_{(S, A) \sim \nu_h} f(S, A) = \bE_{(S, A) \sim \mu_h} [f(S, A) \cdot \nu_h(S, A)/\mu_h(S, A)] \le \conc \cdot \bE_{(S, A) \sim \mu_h} f(S, A)$.
\end{proof}

Since we will make use of \cref{lem:Qeval guarantee} multiple times,
define $\bar \cE$ to be the event (which occurs with probability at least $1 - \delta/2$) on which the two claims of \cref{lem:Qeval guarantee} hold.
An important property that we will use throughout is that since $(\modfc_h)_{h \in [H]}$ are linear function classes, each $q^G \in \Qevalall$, where
\begin{align*}
    \Qevalall 
    \defeq \braces*{
        q \in \bR^{\cS \times \cA} : q_h = \hat T^{\pieval}_G \, q_{\futurestages}, \, \text{ for } h \in [H], G \in \textcolor{purple}{\Gall}, \, q_{H+1} = \vec 0
    }
    \supseteq \Qeval
    \, ,
\end{align*}
is based on $H$ least squares solutions.
Thus for each $h \in [H]$, $q^G_h = \bar q_{\theta_h}$ for some $\theta_h \in \thetaesteval_{G, h}$. 
Before proving our main result (\cref{thm:policy eval}), we state some useful definitions and lemmas.

First, we show that the estimate of expected advantage in \cref{eq:advantage estimate} is close to its expectation under the distribution $\mu_h$.
Define the event
\begin{align}
    \tilde \cE
    &\defeq \bigcap_{h \in [H], a, a' \in \cA} \tilde \cE_{h, a, a'}, \quad \text{where} 
    \label{eq:tilde E defn} \\
     \tilde \cE_{h, a, a'}
    &\defeq \braces*{
        \Bigg|\bigE_{(S, A) \sim \mu_h} \left|q^{G}(S, a) - q^{G}(S, a')\right| 
            - \frac{1}{n} \sum_{j = 1}^n \left|q^{G}(S^j_h, a) - q^{G}(S^j_h, a')\right|\Bigg|
        \le \zeta_2, \,
        \forall q^G \in \Qevalall
        } 
    \, ,
    \nonumber \\
    \zeta_2
    &= \bigOt\left(\log\left(\frac{1}{\alpha}\right)\frac{\sqrt{d} H}{\sqrt{n}}\right), \quad \text{defined in \cref{eq:zeta defn}} \nonumber 
    \, .
\end{align}
\begin{lemma} \label{lem:advantage event probability}
    For any $\delta \in (0,1)$,
    the probability of event $\tilde \cE$ is at least $1 - \delta/2$.
\end{lemma}
\begin{proof}
    Recall that for each $q^G \in \Qevalall$ there exists a $\theta \in \thetaesteval_{G, h}$ such that $q^G_h = \bar q_\theta$, for all $h \in [H]$.
    Since $\thetaesteval_{G, h} \subseteq \cB(\modthetabound)$ the idea will be to show the result for a cover of $\cB(\modthetabound)$, and then relate any $\theta \in \thetaesteval_{G, h}$ back to the cover.

    By \cref{lem:cover number ball}, we know there exists a set 
    \begin{align}
        C_\xi \subset \cB(\modthetabound) \quad \text{with} \quad |C_\xi| = (1 + 2\modthetabound/\xi)^d \quad \text{where} \quad \xi > 0  \, ,
        \label{eq:Theta cover}
    \end{align}
    such that, 
    for any $\theta \in \cB(\modthetabound)$ there exists a $\tilde \theta \in C_\xi$ such that $\norm{\theta - \tilde \theta}_2 \le \xi$.
    Notice that for any $G \in \bbG$, $h \in [H]$ and $\theta \in \thetaeval_{G, h} \subset \cB(\modthetabound)$, 
    there exists a $\tilde \theta \in C_\xi$ such that,
    for any $(s, a) \in \cS_h \times \cA$,
    \begin{align*}
        \left|\bar q_{\theta}(s, a) - \bar q_{\tilde \theta}(s, a) \right|
        &= \left|\clip_{[0, H]} \ip{\phi(s, a)}{\theta} - \clip_{[0, H]} \ip{\phi(s, a)}{\tilde \theta} \right| \\
        &\le \left|\ip{\phi(s, a)}{\theta - \tilde \theta} \right| \\
        &\le \norm{\phi(s, a)}_2 \norm{\theta - \tilde \theta}_2 \\
        &\le \featurebound \xi \, . 
    \end{align*}
    Thus, for any $a, a' \in \cA$,
    \begin{align}
        &\big|\left|\bar q_{\theta}(s, a) - \bar q_{\theta}(s, a')\right| - \left|\bar q_{\tilde \theta}(s, a) - \bar q_{\tilde \theta}(s, a')\right|\big| \nonumber \\
        &\le \big|\left|\bar q_{\theta}(s, a) - \bar q_{\tilde \theta}(s, a)\right| + \left|\bar q_{\tilde \theta}(s, a) - \bar q_{\tilde \theta}(s, a')\right| + \left|\bar q_{\tilde \theta}(s, a') - \bar q_{\theta}(s, a')\right| - 
            \left|\bar q_{\tilde \theta}(s, a) - \bar q_{\tilde \theta}(s, a')\right|\big| \nonumber \\
        &= \left|\bar q_{\theta}(s, a) - \bar q_{\tilde \theta}(s, a)\right| + \left|\bar q_{\tilde \theta}(s, a') - \bar q_{\theta}(s, a')\right|
        \le 2\featurebound \xi \, .
        \label{eq:Gout closeness bound}
    \end{align}
    
    For any $\tilde \theta \in C_\xi$, $h \in [H]$, $a, a' \in \cA$, define the event 
    \begin{align*}
        \check \cE_{\tilde \theta, h, a, a'} 
        &:= \Bigg\{ \abs*{\bigE_{(S, A) \sim \mu_h} \left|\bar q_{\tilde \theta}(S, a) - \bar q_{\tilde \theta}(S, a')\right| 
            - \frac{1}{n} \sum_{j = 1}^n \left|\bar q_{\tilde \theta}(S_h^j, a) - \bar q_{\tilde \theta}(S_h^j, a')\right|} \\
        &\quad \le\frac{H}{\sqrt{n}} \sqrt{\log\left(\frac{4 H |\cA|^2 |C_\xi|}{\delta}\right)} \Bigg\} \, .
    \end{align*}
    Then, since $|\bar q_{\tilde \theta}(s, a) - \bar q_{\tilde \theta}(s, a')| \in [0, H]$ for all $(s, a) \in \cS \times \cA$, 
    by Hoeffding's inequality (\cref{lem:hoeffdings inequality}), we have that, 
    event $\check \cE_{\tilde \theta, h, a, a'}$ occurs with probability at least $1 - \delta/(2H |\cA|^2 |C_\xi|)$.
    Let 
    \begin{align}
        \check \cE_{h, a, a'} := \bigcap_{\tilde \theta \in C_\xi} \check \cE_{\tilde \theta, h, a, a'} \, .
        \label{eq:check event definition}
    \end{align}
    Then, by applying a union bound over $\tilde \theta \in C_\xi$ 
    we have that the probability of $\check \cE_{h, a, a'}$ is at least $1 - \delta/(2H |\cA|^2)$.

    Assume we are on event $\check \cE_{h, a, a'}$. 
    Fix any $q^G \in \Qevalall$,
    and let $\theta \in \thetaesteval_{G, h}$ be such that $q^G_h = \bar q_\theta$.
    Let $\xi > 0$ be a parameter to be chosen later (see \cref{eq:zeta defn}).
    By \cref{eq:Gout closeness bound},
    \begin{align*}
        \abs*{\bigE_{(S, A) \sim \mu_h} \left|\bar q_{\theta}(S, a) - \bar q_{\theta}(S, a')\right|
            - \bigE_{(S, A) \sim \mu_h} \left|\bar q_{\tilde\theta}(S, a) - \bar q_{\tilde\theta}(S, a')\right|} 
        \le 2\featurebound \xi \, .
    \end{align*}
    Since we are on event $\check \cE_{h, a, a'}$,
    \begin{align*}
        &\abs*{\bigE_{(S, A) \sim \mu_h} \left|\bar q_{\tilde \theta}(S, a) - \bar q_{\tilde \theta}(S, a')\right| 
            - \frac{1}{n} \sum_{j = 1}^n \left|\bar q_{\tilde \theta}(S_h^j, a) - \bar q_{\tilde \theta}(S_h^j, a')\right|} \\
        &\le \frac{H}{\sqrt{n}} \sqrt{\log\left(\frac{4H |\cA|^2 |C_\xi|}{\delta}\right)} \, .
    \end{align*}
    By \cref{eq:Gout closeness bound}, 
    \begin{align*}
        \abs*{\frac{1}{n} \sum_{j = 1}^n \left|\bar q_{\tilde\theta}(S_h^j, a) - \bar q_{\tilde\theta}(S_h^j, a')\right| 
            - \frac{1}{n} \sum_{j = 1}^n \left|\bar q_{\theta}(S_h^j, a) - \bar q_{\theta}(S_h^j, a')\right|}
        \le 2\featurebound \xi \, . 
    \end{align*}
    Putting the above results together via three triangle inequalities and recalling that $q^G_h = \bar q_\theta$, we have that with probability at least $1 - \delta/(2H |\cA|^2)$, 
    \begin{align}
        &\abs{\bigE_{(S, A) \sim \mu_h} \left|q^G(S, a) - q^G(S, a')\right|
         - \frac{1}{n} \sum_{j = 1}^n \left|q^G(S_h^j, a) - q^G(S_h^j, a')\right|} \nonumber \\
        &\le \frac{H}{\sqrt{n}} \sqrt{\log\left(\frac{4H |\cA|^2 |C_\xi|}{\delta}\right)} 
            + 4\featurebound \xi \nonumber \\
        &\le \frac{\sqrt{d} H}{\sqrt{n}} \sqrt{\log\left(\frac{4H |\cA|^2 (1 + 8\sqrt{n} \featurebound \modthetabound / (\sqrt{d} H))}{\delta}\right)} + \frac{\sqrt{d} H}{\sqrt{n}}
        \eqdef \zeta_2
        \label{eq:zeta defn}
        \, ,
    \end{align}
    where the last inequality follows by setting $\xi = \sqrt{d} H/(4\sqrt{n}\featurebound)$.
    Notice the above event is exactly $\tilde \cE_{h, a, a'}$ as defined in \cref{eq:tilde E defn}.
    Finally, applying a union bound over $h \in [H]$ and $a, a' \in \cA$ gives that the probability of event $\tilde \cE$ is at least $1 - \delta/2$.
\end{proof}

Next, we show that there exists a $q^G \in \Qeval$ (in particular $q^{\trueG}$)
for which \cref{eq:advantage estimate} is small,
and thus the minimum must be no larger.
\begin{lemma} \label{lem:advantage estimate small}
    On event $\bar \cE \cap \tilde \cE$,
    the minimum over $q^G \in \Qeval$ of \cref{eq:advantage estimate} is at most $\alpha + 2\tilde \eps + \zeta_2$,
    i.e.,
    \begin{align*}
        \min_{q^G \in \Qeval} \max_{h \in [H]} \frac{1}{n} \sum\nolimits_{j=1}^{n} \left|q^G(S_h^j, \pieval) - q^G(S_h^j, \pieval_G)\right| 
        \le \alpha + 2\tilde \eps + \zeta_2
        = \bigOt\left(\alpha + \log\left(\frac{1}{\alpha}\right) \frac{\conc^{3/2} H^{5/2} d^{3/2}}{\sqrt{n}}\right)
        \, .
    \end{align*}
\end{lemma}
\begin{proof} \label{proof:lem:Gout nonempty}
    On event $\tilde \cE$ 
    for all $h \in [H]$, 
    \begin{align*}
        \frac{1}{n} \sum_{j = 1}^n \left|q^{\trueG}(S^j_h, \pieval) - q^{\trueG}(S^j_h, \pieval_{\trueG})\right|
        \le \bigE_{(S, A) \sim \mu_h} \left|q^{\trueG}(S, \pieval) - q^{\trueG}(S, \pieval_{\trueG})\right| 
            + \zeta_2
            \, .
    \end{align*}
    On event $\bar \cE$, by the second claim in \cref{lem:Qeval guarantee}, $q^{\trueG} \in \Qeval \subseteq \Qevalall$.
    Then, by the first claim in \cref{lem:Qeval guarantee}, and two triangle inequalities, for all $h \in [H]$,
    \begin{align*}
        \bigE_{(S, A) \sim \mu_h} \left|q^{\trueG}(S, \pieval) - q^{\trueG}(S, \pieval_{\trueG})\right| 
        \le \bigE_{(S, A) \sim \mu_h} \left|q^{\pieval_{\trueG}}(S, \pieval) - q^{\pieval_{\trueG}}(S, \pieval_{\trueG})\right| 
            + 2\tilde \eps \, .
    \end{align*}
    Since $\pieval_{\trueG}$ only differs from $\pieval$ on states $s$ where $\range(s) \le \alpha$, and for any such state, $\max_{a, a' \in \cA} q^{\pieval_{\trueG}}(S, a) - q^{\pieval_{\trueG}}(S, a') \le \alpha$ (by the definition of $\range$), we have that
    \begin{align*}
        &\bigE_{(S, A) \sim \mu_h} \left|q^{\pieval_{\trueG}}(S, \pieval) - q^{\pieval_{\trueG}}(S, \pieval_{\trueG})\right| \\
        &= \bigE_{(S, A) \sim \mu_h} \left[\I{\range(S) \le \alpha} \left|q^{\pieval_{\trueG}}(S, \pieval) - q^{\pieval_{\trueG}}(S, \pieval_{\trueG})\right| \right]
        \le \alpha \, .
    \end{align*}
    Putting everything together and noting that 
    the minimum over $q^G \in \Qeval$ of \cref{eq:advantage estimate} is at most the value at $q^{\trueG}$ (since on event $\bar \cE$, $q^{\trueG} \in \Qeval$),
    we get the desired result.

\end{proof}

We state one final lemma, which makes use of the above lemmas, before proving \cref{thm:policy eval}.

\begin{lemma} \label{lem:advantage bound}
    On event $\bar \cE \cap \tilde \cE$,
    for all admissible distributions $\nu = (\nu_t)_{t \in [H]}$,
    and for all $h \in [H]$,
    \begin{align*}
        &\bigE_{(S, A) \sim \nu_h} \left| q^{\pieval_{\outG}}(S, \pieval) - q^{\pieval_{\outG}}(S, \pieval_{\outG}) \right| 
        \le \tilde \alpha
        = \bigOt\left(\alpha + \log\left(\frac{1}{\alpha}\right) \frac{\conc^{5/2} H^{5/2} d^{3/2}}{\sqrt{n}}\right) \, .
    \end{align*}
\end{lemma}
\begin{proof} \label{proof:lem:advantage bound}
    By two triangle inequalities and \cref{lem:Qeval guarantee}, on event $\bar \cE$, for any admissible distribution $\nu = (\nu_t)_{t \in [H]}$, and for all $h \in [H]$,
    \begin{align}
        &\bigE_{(S, A) \sim \nu_h} \left| q^{\pieval_{\outG}}(S, \pieval) - q^{\pieval_{\outG}}(S, \pieval_{\outG}) \right| \nonumber \\
        &\le \bigE_{(S, A) \sim \nu_h} \left|q^{\outG}(S, \pieval) - q^{\outG}(S, \pieval_{\outG})\right|
            + 2\tilde \eps \, . 
        \label{eq:policy to estimates}
    \end{align}
    We bound the expectation on the right side next. 
    Since $q^{\outG} \in \Qeval \subset \Qevalall$, on event $\tilde \cE$, for all $h \in [H]$,
    \begin{align*}
        \bigE_{(S, A) \sim \mu_h} \left|q^{\outG}(S, \pieval) - q^{\outG}(S, \pieval_{\outG})\right| 
        \le \frac{1}{n} \sum_{j = 1}^n \left|q^{\outG}(S^j_h, \pieval) - q^{\outG}(S^j_h, \pieval_{\outG})\right|
            + \zeta_2 \, ,
    \end{align*}
    and by \cref{lem:advantage estimate small}, on event $\tilde \cE \cap \bar \cE$,
    \begin{align*}
        \max_{h \in [H]}\frac{1}{n} \sum_{j = 1}^n \left|q^{\outG}(S^j_h, \pieval) - q^{\outG}(S^j_h, \pieval_{\outG})\right| 
        \le \alpha + 2\tilde \eps + \zeta_2 \, .
    \end{align*}
    Combining the above two results,
    on event $\tilde \cE \cap \bar \cE$, for all $h \in [H]$,
    \begin{align*}
        \bigE_{(S, A) \sim \mu_h} \left|q^{\outG}(S, \pieval) - q^{\outG}(S, \pieval_{\outG})\right|
        \le \alpha + 2\tilde \eps + 2\zeta_2 \, .
    \end{align*}
    Notice that since the expectation is over a non-negative function (due to the absolute value) we can apply \cref{lem:change of measure} to get that 
    for all admissible distributions $\nu = (\nu_t)_{t \in [H]}$, and for all $h \in [H]$,
    \begin{align*}
        \bigE_{(S, A) \sim \nu_h} \left|q^{\outG}(S, \pieval) - q^{\outG}(S, \pieval_{\outG})\right|
        \le \conc \cdot (\alpha + 2\tilde \eps + 2\zeta_2) 
        \eqdef \tilde \alpha
        \, .
    \end{align*}
    Combining the above with \cref{eq:policy to estimates} completes the proof.

\end{proof}

We are now ready to prove our main result (\cref{thm:policy eval}).
\begin{proof}[of \cref{thm:policy eval}] \label{proof:policy eval}
    Recall that $\outG \in \Geval$ is such that $q^{\outG} \in \Qeval$ is the minimizer of \cref{eq:advantage estimate}.
    Our goal is to show that 
    \begin{align*}
        \left| v^{\pieval}(\startstate) - \vout \right| 
        = \left| v^{\pieval}(\startstate) - q^{\outG}(\startstate, \pieval_{\outG}) \right| 
        \le \eps \, .
    \end{align*}
    We relate things to $v^{\pieval_{\outG}}(\cdot) = q^{\pieval_{\outG}}(\cdot, \pieval_{\outG})$: 
    \begin{align*}
        \left| v^{\pieval}(\startstate) - q^{\outG}(\startstate, \pieval_{\outG}) \right| 
        &\le \left| v^{\pieval}(\startstate) - v^{\pieval_{\outG}}(\startstate) \right| 
            + \left| q^{\pieval_{\outG}}(\startstate, \pieval_{\outG}) - q^{\outG}(\startstate, \pieval_{\outG}) \right| \, .
    \end{align*}
    By \cref{lem:Qeval guarantee} (since $\bPmarg{h}_{\pieval_{\outG}, s_1}$ is an admissible distribution), on event $\bar \cE$,
    \begin{align*}
        \left| q^{\pieval_{\outG}}(\startstate, \pieval_{\outG}) - q^{\outG}(\startstate, \pieval_{\outG}) \right|
        = \bigE_{(S_1, A_1) \sim \bPmarg{1}_{\pieval_{\outG}, s_1}} \left| \pieval_{\outG}(S_1, A_1) - q^{\outG}(S_1, A_1) \right|
        \le \tilde\eps \, .
    \end{align*}
    To bound the first term we will use \cref{eq:pdl evallearner}, which we restate here for clarity: 
    \begin{align*}
        \left| v^{\pieval}(\startstate) - v^{\pieval_{\outG}}(\startstate) \right|
        &= \left| \sum_{h=1}^H \bigE_{(S_h, A_h) \sim \bPmarg{h}_{\pieval, s_1}} \left(q^{\pieval_{\outG}}(S_h, \pieval) - q^{\pieval_{\outG}}(S_h, \pieval_{\outG})\right) \right|
        \, . 
    \end{align*}
    Then by $H$ triangle inequalities and Jensen's inequality we have that:
    \begin{align*}
        \left| \sum_{h=1}^H \bigE_{(S_h, A_h) \sim \bPmarg{h}_{\pieval, s_1}} \left(q^{\pieval_{\outG}}(S_h, \pieval) - q^{\pieval_{\outG}}(S_h, \pieval_{\outG})\right) \right| 
        &\le \sum_{h=1}^H \bigE_{(S_h, A_h) \sim \bPmarg{h}_{\pieval, s_1}} \left|q^{\pieval_{\outG}}(S_h, \pieval) - q^{\pieval_{\outG}}(S_h, \pieval_{\outG})\right|
        \, .
    \end{align*}
    Noting that $\bPmarg{h}_{\pieval, s_1}$ is an admissible distribution, and applying \cref{lem:advantage bound},
    we have that on event $\bar \cE \cap \tilde \cE$,
    \begin{align*}
        \left| v^{\pieval}(\startstate) - v^{\pieval_{\outG}}(\startstate) \right|
        \le H \tilde \alpha 
        \, .
    \end{align*}
    By combining the above two bounds, 
    \begin{align}
        \left| v^{\pieval}(\startstate) - \vout \right| 
        \le H \tilde \alpha + \tilde \eps
        = \bigOt\left(H \alpha + \log(1/\alpha) \frac{\conc^{5/2} H^{7/2} d^{3/2}}{\sqrt{n}}\right)
        \, .
        \label{eq:eval final bound}
    \end{align}
    If we set 
    \begin{align}
        \alpha = \bigOt\left(\frac{\conc^{5/2} H^{5/2} d^{3/2}}{\sqrt{n}}\right) \, ,
        \label{eq:alpha setting evallearner}
    \end{align}
    then if $n = \tilde \Theta(\frac{\conc^5 H^7 d^3}{\eps^2})$%
    , it implies that 
    \begin{align*}
        \left| v^{\pieval}(\startstate) - \vout \right| 
        \le \eps \, .
    \end{align*}
    Since the probability of event $\bar \cE$ is at least $1 - \delta/2$ (\cref{lem:Qeval guarantee}), 
    and the probability of event $\tilde \cE$ is at least $1 - \delta/2$ (\cref{lem:advantage event probability}),
    the desired result follows by a union bound.
\end{proof}

\section{Analysis of \optlearner} \label{sec:analysis of optlearner}
The following analysis is basically identical to that in \citet{tkachuk2024trajectory}[Theorem 1].
However, we get an improved bound on $n$ due to our improved bound in \cref{lem:Qopt guarantee}.

We begin by proving \cref{lem:skippy policy error}.
\begin{proof}[of \cref{lem:skippy policy error}] \label{proof:skippy policy error}
    By the performance difference lemma (\cref{lem:performance difference lemma}), for any $h \in [H+1]$, 
    $s \in \cS_h$,
    and $\pi \in \Pi$,
    \begin{align}
        \left| v^{\pi}(s) - v^{\pi_{\trueG}}(s) \right|
        = \abs{\textstyle\sum\nolimits_{t=h}^H \bE_{(S_t, A_t) \sim \bPmarg{t}_{\pi, s}} \left(q^{\pi_{\trueG}}(S_t, \pi) - q^{\pi_{\trueG}}(S_t, \pi_{\trueG})\right)} 
        \le H \alpha
        \label{eq:pdl skippy eval policy} \, .
    \end{align}
    The inequality follows from the fact that $\pi$ and $\pi_{\trueG}$ only differ on states $s$ with $\range(s) \le \alpha$,
    which shows the first result of the lemma.

    To show the second result
    we will show that for any $h \in [H]$, $s \in \cS_h$,
    \begin{align*}
        v^{\piopt}(s) - v^{\pioptg_{\trueG}}(s)
        \le v^{\piopt}(s) - v^{\piopt_{\trueG}}(s) \, .
    \end{align*}
    Then by applying first result we get second result.

    It will be sufficient to show $v^{\pioptg_{\trueG}} \ge v^{\piopt_{\trueG}}$.
    The proof is taken verbatim from \citep{tkachuk2024trajectory}[Lemma D.1].
    We use induction.
    The base case of $h=H+1$ is immediately true by definition as $v$-values are $0$ on $\termstate$, regardless of the policy.
    Assuming the inductive hypothesis holds for $h+1$, we continue by proving it for $h$.
    Let $(s,a)\in\cS_h\times\cA$ be arbitrary.
    Notice that 
    \[
    q^{\pioptg_{\trueG}}(s,a) - q^{\piopt_{\trueG}}(s,a) = \bigE_{S' \sim P(s,a)} v^{\pioptg_{\trueG}}(S') - v^{\piopt_{\trueG}}(S') \ge 0 \, ,
    \]
    where the inequality is due to the inductive hypothesis.
    Next, for any $s\in\cS_h$, by the above and the definition of the policies,
    \begin{align*}
    v^{\pioptg_{\trueG}}(s)
    &=
    \omega_{\trueG}(s) q^{\pioptg_{\trueG}}(s, \pibehave) + (1-\omega_{\trueG}(s)) \max_{a\in\cA} q^{\pioptg_{\trueG}}(s,a) \\
    &\ge \omega_{\trueG}(s) q^{\piopt_{\trueG}}(s, \pibehave) + (1-\omega_{\trueG}(s)) \max_{a\in\cA} q^{\piopt_{\trueG}}(s,a) \\
    &\ge \omega_{\trueG}(s) q^{\piopt_{\trueG}}(s, \pibehave) + (1-\omega_{\trueG}(s)) q^{\piopt_{\trueG}}(s, \piopt) 
    = 
    v^{\piopt_{\trueG}}(s)\,,
    \end{align*}
    finishing the induction.
\end{proof}

Now, we show that for any policy $\pi$ the value function of a policy $\pi'$ that is greedy with respect to a good estimate of $q^\pi$ cannot be much worse than $\pi$.
\begin{lemma} \label{lem:policy improvement}
    Let $\pi, \pi'$ be any two policies.
    Let $q \in \bR^{\cS \times \cA}$ be such that for any admissible distribution $\nu = (\nu_t)_{t \in [H]}$, and for all $h \in [H]$,
    \begin{align*}
        \bigE_{(S, A) \sim \nu_h} \left| q^\pi(S, A) - q(S, A) \right| 
        \le \kappa \, .
    \end{align*}
    If $\pi'(s) = \argmax_{a \in \cA} q(s, a)$ for all $s \in \cS$,
    then for any admissible distribution $\nu = (\nu_t)_{t \in [H]}$,
    and for all $h \in [H+1]$, 
    \begin{align*}
        \bigE_{(S_h, A) \sim \nu_h} \brackets{v^{\pi}(S_h) - v^{\pi'}(S_h)} 
        \le 2(H-h+1)\kappa 
        \, .
    \end{align*}
\end{lemma}
\begin{proof}
    We use induction.
    The base case of $h=H+1$ is immediately true by definition as $v$-values are $0$ on $\termstate$, regardless of the policy.
    Assuming the inductive hypothesis holds for $h+1$, we continue by proving it for $h$.
    Adding and subtracting $q$,
    \begin{align*}
        \bigE_{(S_h, A) \sim \nu_h} \brackets{v^{\pi}(S_h) - v^{\pi'}(S_h)} 
        = \bigE_{(S_h, A) \sim \nu_h} \brackets{v^{\pi}(S_h) - q(S_h, \pi)} 
        + \bigE_{(S_h, A) \sim \nu_h}\brackets{q(S_h, \pi) - v^{\pi'}(S_h)} 
        \, .
    \end{align*}
    The first term is bounded by $\kappa$ by the assumption of the lemma.
    To bound the second term, notice that by the definition of $\pi'$, $q(S_h, \pi) \le q(S_h, \pi')$.
    Thus, by making use of the assumption of the lemma and the inductive hypothesis,
    \begin{align*}
        \bigE_{(S_h, A) \sim \nu_h}\brackets{q(S_h, \pi) - v^{\pi'}(S_h)} 
        &\le \bigE_{(S_h, A) \sim \nu_h}\brackets{q(S_h, \pi') - q^{\pi'}(S_h, \pi')} \\
        &\le \bigE_{(S_h, A) \sim \nu_h}\brackets{q^\pi(S_h, \pi') - q^{\pi'}(S_h, \pi')} 
            + \kappa \\
        &\le \bigE_{(S_h, A) \sim \nu_h}\brackets*{\bigE_{S_{h+1} \sim P(\cdot|S_h, A)} \brackets{v^\pi(S_{h+1}) - v^{\pi'}(S_{h+1})}} 
            + \kappa \\
        &\le 2(H - h) \kappa  + \kappa
        \, .
    \end{align*}
    Putting everything together we get the desired result.
\end{proof}

We are now ready to prove \cref{thm:policy opt} . 
\begin{proof}[of \cref{thm:policy opt}] \label{proof:policy opt}
    We need to show that 
    \begin{align*}
        v^{\piopt}(\startstate) - v^{\piout}(\startstate) \le \eps \, .
    \end{align*}
    By \cref{lem:skippy policy error},
    \begin{align*}
        v^{\piopt}(\startstate) - v^{\piout}(\startstate) 
        \le v^{\pioptg_{\trueG}}(\startstate) - v^{\piout}(\startstate) 
        + H \alpha 
        \, .
    \end{align*}
    Recall that $\piout(s) = \argmax_a \hat q(s, a)$, where $\hat q = \argmax_{q \in \Qopt} \max_a q(\startstate, a)$.
    Adding and subtracting $\hat q$,
    \begin{align}
        v^{\pioptg_{\trueG}}(\startstate) - v^{\piout}(\startstate) 
        \le v^{\pioptg_{\trueG}}(\startstate) - \hat q(\startstate, \piout) + \hat q(\startstate, \piout) - v^{\piout}(\startstate)
        \, .
        \label{eq:opt proof decomp}
    \end{align}
    For the remainder of the proof, assume we are on the event (which occurs with probability at least $1 - \delta$) for which the two claims of \cref{lem:Qopt guarantee} hold.

    By the second claim of \cref{lem:Qopt guarantee}, $q^{\trueG} \in \Qopt$.
    Thus, by the first claim of \cref{lem:Qopt guarantee}, 
    \begin{align*}
        v^{\pioptg_{\trueG}}(\startstate)
        = q^{\pioptg_{\trueG}}(\startstate, \pioptg_{\trueG})
        \le q^{\trueG}(\startstate, \pioptg_{\trueG}) + \tilde \eps 
        \, .
    \end{align*}
    Since $\hat q$ is the maximizer of $\max_a q(\startstate, a)$ 
    over $q \in \Qopt$, and $\max_a \hat q(\startstate, a) = \hat q(\startstate, \piout)$ (by definition of $\piout$), 
    \begin{align*}
        q^{\trueG}(\startstate, \pioptg_{\trueG}) 
        \le \hat q(\startstate, \piout) 
        \, .
    \end{align*}
    Putting the above two inequalities together we get that following bound on the first term in \cref{eq:opt proof decomp}:
    \begin{align*}
        v^{\pioptg_{\trueG}}(\startstate) - \hat q(\startstate, \piout) 
        \le \tilde \eps \, .
    \end{align*}
    
    Next we bound the second term in \cref{eq:opt proof decomp}.
    Let $\bar G \in \Gopt$ be the modification for which $\hat q$ is based on.
    By the first claim of \cref{lem:Qopt guarantee}, 
    \begin{align*}
        \hat q(\startstate, \piout) 
        \le q^{\pioptg_{\bar G}}(\startstate, \piout) + \tilde \eps 
        \, .
    \end{align*}
    Also,
    \begin{align*}
        q^{\pioptg_{\bar G}}(\startstate, \piout) - v^{\piout}(\startstate)
        = \bigE_{(S_2, A_2) \sim \bPmarg{2}_{\piout, s_1}} \brackets{v^{\pioptg_{\bar G}}(S_2) - v^{\piout}(S_2)} \, .
    \end{align*}
    Since (by the first claim of \cref{lem:Qopt guarantee}), for all admissible distributions $\nu = (\nu_t)_{t \in [H]}$, and for all $h \in [H]$,
    \begin{align*}
        &\bigE_{(S, A) \sim \nu_h} \left| q^{\bar G}(S, A) - q^{\pioptg_{\bar G}}(S, A) \right| 
        \le \tilde \eps 
        \, ,
    \end{align*}
    we can apply \cref{lem:policy improvement} to get that
    \begin{align*}
        \bigE_{(S_2, A_2) \sim \bPmarg{2}_{\piout, s_1}} \brackets{v^{\pioptg_{\bar G}}(S_2) - v^{\piout}(S_2)} 
        \le 2(H-1) \tilde \eps
        \, .
    \end{align*}
    Putting the above three results together we get the following bound on the second term in \cref{eq:opt proof decomp}: 
    \begin{align*}
        \hat q(\startstate, \piout) - v^{\piout}(\startstate)
        \le 2(H-1) \tilde \eps + \tilde \eps
        \, .
    \end{align*}
        
    Putting everything together we get that 
    \begin{align}
        v^{\pioptg_{\trueG}}(\startstate) - v^{\piout}(\startstate) 
        \le H \alpha + 2H \tilde \eps 
        = \bigOt\left(H \alpha + \log(1/\alpha) \frac{\conc^{3/2} H^{7/2} d^{3/2}}{\sqrt{n}}\right)
        \, .
        \label{eq:opt final bound}
    \end{align}
    If we set 
    \begin{align}
        \alpha \defeq \bigOt\left(\frac{\conc^{3/2} H^{5/2} d^{3/2}}{\sqrt{n}}\right) \, ,
        \label{eq:alpha setting optlearner}
    \end{align}
    then if $n = \tilde \Theta(\frac{\conc^3 H^7 d^3}{\eps^2})$%
    , it implies that 
    \begin{align*}
        v^{\piopt}(\startstate) - v^{\piout}(\startstate) 
        \le \eps \, .
    \end{align*}
\end{proof}

\section{Missing Proofs and Definitions from \Cref{sec:learner}} \label{sec:learner missing proofs and defs}

\begin{definition}[Definition F.1 in \citep{weisz2023online}]\label{def:nearopt}
A finite set $G\subset \R^d$ is the basis of a near-optimal design for a set $\Theta\subseteq\R^d$, if there exists a probability distribution $\rho$ over elements of $G$, such that
for any $\theta\in\Theta$,
\begin{align}
&\ip{v}{\theta} = 0 \quad\text{for all } v\in\kernel(V(G, \rho)), \text{ and}\label{eq:v-max-rank}\\
&\norm{\theta}_{V(G, \rho)^\dag}^2\le 2d, \label{eq:near-opt-2d}\\
&\text{where } V(G, \rho)=\sum_{x\in G} \rho(x) xx^\top\,,\label{eq:design-matrix}
\end{align}
where for a matrix $X$, $X^\dag$ denotes the Moore--Penrose pseudoinverse of $X$, and $\kernel(X)$ denotes its kernel (null space).
\end{definition}

\subsection{Proof of \cref{lem:qpi realizability implies bellman completeness}}
    \label{proof:qpi realizability implies bellman completeness}
\begin{proof}
    The following lemma from \citep{tkachuk2024trajectory} will be useful (proof: \cref{proof:approx linear MDP neurips})%
    \footnote{
        The actual lemma in \citep{tkachuk2024trajectory} is stated for the case with misspecification error;
        however, since in our case the misspecification error is $0$, we state the lemma for that case only.
    }.
    \begin{lemma}[Lemma 4.2 in \citep{tkachuk2024trajectory}] \label{lem:approx linear MDP neurips}
        For any $f:\cS\to[0,H]$ 
        with $f(\termstate)=0$,
        policy $\pi \in \Pi$,
        and stage $h \in [H]$, there exists
        a parameter $\rho_h^\pi(f) \in \cB(\modthetabound)$
        such that for all $(s, a) \in \cS_h \times \cA$,
        \begin{align*}
            \bigE_{\trajectoryrand \sim \bP_{\pi, s, a}} \sum_{\tau=h+1}^{H+1} (R_{h:\tau-1} + f_\tau(S_\tau)) \, F_{\trueG, \trajectoryrand, h+1}(\tau)
            = \ip{\phi(s,a)}{\rho_h^\pi(f)} 
            \, .
        \end{align*}
    \end{lemma}
    Set the policy in the above lemma to the behavior policy $\pibehave$.
    Notice that for any $\rho_h^{\pibehave}(f) \in \cB(\modthetabound)$ in the lemma $\ip{\phi(\cdot, \cdot)}{\rho_h^{\pibehave}(f)} \in \modfc_h$.
    Also, notice that for any $h \in [H]$, $q_{\futurestages}:[0,H]^{\cS\!\cA_{\futurestages}}$ with $q_{H+1}(\termstate, \cdot) = 0$,
    \begin{align*}
        T_{\trueG}^{\pi} \, q_{\futurestages} (s, a)
        &= \bigE_{\trajectoryrand \sim \bP_{\pibehave, s, a}} g_{\trueG}^{\pi}(q_{\futurestages}, \trajectoryrand) \\
        &= \bigE_{\trajectoryrand \sim \bP_{\pibehave, s, a}} \sum_{\tau=h+1}^{H+1} (R_{h:\tau-1} + q_\tau(S_\tau, \pi)) \, F_{\trueG, \trajectoryrand, h+1}(\tau)
        \, ,
    \end{align*}
    which matches the left-hand side of the equality in the above lemma with $f_t(\cdot) = q_t(\cdot, \pi)$ for all $t \in [h:H+1]$.
    Thus, by the above lemma,
    \begin{align*}
        T_{\trueG}^{\pi} \, q_{\futurestages} 
        \in \modfc_h \, .
    \end{align*} 
    Similarly, if we let $f_t(\cdot) = \max_a q_t(\cdot, a)$ for all $t \in [h:H+1]$, then by the above lemma,
    \begin{align*}
        T_{\trueG} \, q_{\futurestages} \in \modfc_h \, .
    \end{align*} 
    This completes the proof of \cref{lem:qpi realizability implies bellman completeness}. 
\end{proof}

We provide the proof of \cref{lem:approx linear MDP neurips} for completeness, and since we found some minor errors in the proof of the original lemma in \citep{tkachuk2024trajectory}, 
after correcting them, the size of $\modthetabound$ increases by a factor of $\sqrt{2d}$.

\begin{proof}[of \cref{lem:approx linear MDP neurips}]
    \label{proof:approx linear MDP neurips}
    We first mention a useful definition and lemma from \citet{weisz2023online}.
    \begin{definition}[$\bar \alpha$-admissible function (Definition 4.6 in \citet{weisz2023online})] \label{def:good function}
        For any $h \in [H], f: \cS_h \to \bR$ is $\bar \alpha$-admissible for some $\bar \alpha \ge 0$ if for all $s \in \cS_h$, $\bar \alpha \abs{f(s)} \le \range_G(s)$.
    \end{definition}
    \begin{lemma}[Realizability of admissible functions (Lemma 4.7 in \citet{weisz2023online})] \label{lem:good function realizability}
        For any $t \in [2:H]$, $h \in [t-1]$, $\bar \alpha \ge 0$, if $f: \cS_t \to \bR$ is $\bar \alpha$-admissible, then for any $\pi \in \Pi$,
        there exists a $\tilde \theta \in \cB(4 d_0 \thetabound / \bar \alpha)$ such that for all $(s, a) \in \cS_h \times \cA$,
        \begin{align*}
            \bE_{\trajectoryrand \sim \bP_{\pi, s, a}} f(S_t) 
            = \ip{\phi(s, a)}{\tilde \theta} \, . 
        \end{align*}
    \end{lemma}

    Now we prove the lemma.
    Fix any $f:\cS\to[0,H]$ 
    with $f(\termstate)=0$,
    $\pi \in \Pi$.
    For $h \in [2:H+1]$, define $\tilde g_h : \cS_h\to [-H,H]$
    as 
    $\tilde g_{H+1}(\cdot)=0$, and for all $h \in [2:H]$, $s \in \cS_h$,
    \begin{align*}
    \tilde g_h(s)
    &\defeq \bigE_{\trajectoryrand \sim \bP_{\pi, s}} \brackets*{\sum_{\tau=\textcolor{purple}{h}}^{H+1} (-R_{\tau:H} + f_\tau(S_\tau)) \, F_{\trueG, \trajectoryrand, \textcolor{purple}{h}}(\tau)}
    \, .
    \end{align*}
    Notice that for any $h \in [H]$, $s \in \cS_h$, and $a \in \cA$,
    \begin{align}
    \bigE_{\trajectoryrand \sim \bP_{\pi, s, a}} \brackets*{\sum_{\tau=h+1}^{H+1} (R_{h:\tau-1} + f_\tau(S_\tau)) \, F_{\trueG, \trajectoryrand, h+1}(\tau)}
    &= \bigE_{\trajectoryrand \sim \bP_{\pi, s, a}} \brackets*{\tilde g_{h+1}(S_{h+1}) + R_{h:H}} \\
    &= \bigE_{\trajectoryrand \sim \bP_{\pi, s, a}} \brackets*{\tilde g_{h+1}(S_{h+1})} + q^\pi_h(s, a) 
    \label{eq:target-into-q-and-tilde-g}
    \, .
    \end{align}
    The second term of the sum, $q^\pi_h(\cdot, \cdot)$ is equal to $\ip{\phi(\cdot, \cdot)}{\theta^*_h(\pi)}$ by \cref{ass:q-pi realizability}.
    The first term needs more work before \cref{lem:good function realizability} can be applied.
    First we can write $g_h$ in a different form. 
    For all $h \in [2:H]$, $s \in \cS_h$,
    \begin{align*}
        \tilde g_h(s) 
        &= \bigE_{\trajectoryrand \sim \bP_{\pi, s}} \brackets*{\sum_{\tau=\textcolor{purple}{h}}^{H+1} (-R_{\tau:H} + f_\tau(S_\tau)) \, F_{\trueG, \trajectoryrand, \textcolor{purple}{h}}(\tau)} \\
        &= \bigE_{\trajectoryrand \sim \bP_{\pi, s}} \brackets*{\sum_{\tau=\textcolor{purple}{h}}^{H+1} (-R_{\tau:H} + f_\tau(S_\tau)) \paren*{1-\omega_{\trueG}(S_\tau)} \prod_{u=\textcolor{purple}{h}}^{\tau-1} \omega_{\trueG}(S_u)} \\
        &= \left(1-\omega_{\trueG}(s)\right) \bigE_{\trajectoryrand \sim \bP_{\pi, s}} \brackets*{-R_{h:H} + f_h(s)} \\ 
        &\qquad+ \omega_{\trueG}(s)\bigE_{\trajectoryrand \sim \bP_{\pi, s}} \brackets*{\sum_{\tau=\textcolor{purple}{h+1}}^{H+1} (-R_{\tau:H} + f_\tau(S_\tau)) \paren*{1-\omega_{\trueG}(S_\tau)} \prod_{u=\textcolor{purple}{h+1}}^{\tau-1} \omega_{\trueG}(S_u)}\\
        &= \left(1-\omega_{\trueG}(s)\right) \bigE_{\trajectoryrand \sim \bP_{\pi, s}} \brackets*{-R_{h:H} + f_h(s)} + \omega_{\trueG}(s)\bigE_{\trajectoryrand \sim \bP_{\pi, s}} \brackets*{\tilde g_{h+1}(S_{h+1})} \, .
    \end{align*}
    For $h\in[2:H]$ define $g_h : \cS_h\to \R$ as
    \begin{align*}
        g_h(s) = \left(1-\omega_{\trueG}(s)\right) \bigE_{\trajectoryrand \sim \bP_{\pi, s}} \left[-R_{h:H} + f_h(s) - \tilde g_{h+1}(S_{h+1})\right] \, .
    \end{align*}
    Then
    \begin{align*}
        \tilde g_h(s) = g_h(s) + \bigE_{\trajectoryrand \sim \bP_{\pi, s}} \brackets*{\tilde g_{h+1}(S_{h+1})} \, .
    \end{align*}
    Thus $\tilde g_h$ can be decomposed into a sum of $g_t$ functions as for all $h\in[2:H]$, $s\in\cS_h$,
    \begin{align}\label{eq:tilde-g-decompose}
    \tilde g_h(s) = 
    \bigE_{\trajectoryrand \sim \bP_{\pi, s}} \brackets*{\sum_{t=h}^H g_t(S_t)}
    \, .
    \end{align}
    The benefit of decomposing $\tilde g_h$ into $g_t$ functions is that $g_t$ are $\alpha/(2H\sqrt{2d})$-admissible (recall \cref{def:good function}).
    To see this, note that for any trajectory, $t \in [2:H]$, $s \in \cS_t$,
    \begin{align*}
        -R_{t:H} + f_t(s) - \tilde g_{t+1}(S_{t+1})\in[-2H,2H] \, .
    \end{align*}
    Then, $g_t(s)$ takes the expected value (which is still in $[-2H,2H]$) of the above display and multiplies it by $1-\omega_{\trueG}(s)$ (which by \cref{eq:omega} is in $[0,1]$).
    Notice that 
    \begin{align*}
        1 - \omega_{\trueG}(s)
        \le \frac{\sqrt{2d} \cdot \range_{\trueG}(s)}{\alpha}
        \, .
    \end{align*}
    This holds since: when $\range_{\trueG}(s) \ge \alpha/\sqrt{2d}$ the right hand side is greater than $1$, 
    when $\range_{\trueG}(s) \le \alpha/(2\sqrt{2d})$, $\omega_{\trueG}(s) = 1$,
    and when $\alpha/(2\sqrt{2d}) \le \range_{\trueG}(s) < \alpha/\sqrt{2d}$, $1 - \omega_{\trueG}(s) = 2 \sqrt{2d} \cdot \range_{\trueG}(s)/\alpha - 1 \le \sqrt{2d} \cdot \range_{\trueG}(s)/\alpha$.
    Therefore $g_t$ is $\alpha/(2H\sqrt{2d})$-admissible.
    By using \cref{lem:good function realizability} we get that for any $h\in[H-1]$,  
    $t\in[h+1:H]$,
    there exists a $\tilde \theta_t \in \cB(8 H d_0 \sqrt{2d} \thetabound / \alpha)$ such that for all $(s, a) \in \cS_h \times \cA$,
    \begin{align*}
        \bigE_{\trajectoryrand \sim \bP_{\pi, s, a}} g_{t}(S_{t}) 
        = \ip{\phi(s, a)}{\tilde \theta_t} \, .
    \end{align*}
    By combining this with \cref{eq:tilde-g-decompose}, for all $h \in [H]$,
    the parameter $\tilde \theta = \sum_{t=h+1}^H \tilde \theta_t$ is in $\cB(8 H^2 d_0 \sqrt{2d} \thetabound / \alpha)$, and for all $(s, a) \in \cS_h \times \cA$,
    \begin{align*}
        \bigE_{\trajectoryrand \sim \bP_{\pi, s, a}} \tilde g_{h+1}(S_{h+1})
        = \ip{\phi(s, a)}{\tilde \theta} \, .
    \end{align*} 
    Combining the above with \cref{eq:target-into-q-and-tilde-g}, we have that for all $h\in[H]$,
    the parameter $\theta = \tilde \theta + \theta^*_h(\pi)$ is in $\cB(8 H^2 d_0 \sqrt{2d} \thetabound / \alpha + \thetabound)$, and for all $(s, a) \in \cS_h \times \cA$,
    \begin{align*}
        \bigE_{\trajectoryrand \sim \bP_{\pi, s, a}} \brackets*{\sum_{\tau=h+1}^{H+1} (R_{h:\tau-1} + f_\tau(S_\tau)) \, F_{\trueG, \trajectoryrand, h+1}(\tau)}
        &= \bigE_{\trajectoryrand \sim \bP_{\pi, s, a}} \brackets*{\tilde g_{h+1}(S_{h+1})} + q^\pi_h(s, a) \\
        &= \ip{\phi(s, a)}{\theta}
        \, .
    \end{align*}
    To finish the proof we define $\rho_h^\pi(f) = \theta$.
 
\end{proof}

\subsection{Proof of \cref{lem:skippy bellman policy operator backup}}
    \label{proof:skippy bellman policy operator backup}
\begin{proof}
    We need to show that $q^{\pi_G}_h = T_G^{\pi} \, q_{\futurestages}^{\pi_G}$,
    for any policy $\pi \in \Pi$ and modification $G \in \Gall$.
    We make use of the fact that $q^{\pi_G}_h = T^\pi q^{\pi_G}_{h+1}$, and the definition of $\pi_G$,
    to expand $q^{\pi_G}_h$ as follows, for any $(s, a) \in \cS_h \times \cA$, $h \in [H]$:
    \begin{align*}
        q^{\pi_G}_h(s, a)
        &= T^{\pi_G} q^{\pi_G}_{h+1} \\
        &= \bigE_{\trajectoryrand \sim \bP_{\pi_G, s, a}} \brackets{R_h + q^{\pi_G}_{h+1}(S_{h+1}, \pi_G)} \\
        &= \bigE_{\trajectoryrand \sim \bP_{\pi_G, s, a}} \brackets{R_h + (1-\omega_G(S_{h+1})) q^{\pi_G}_{h+1}(S_{h+1}, \pi) + \omega_G(S_{h+1}) q^{\pi_G}_{h+1}(S_{h+1}, \pibehave)}
        \, .
    \end{align*}
    Notice that if we change the expectation to be with respect to $\bP_{\pibehave, s, a}$ instead of $\bP_{\pi_G, s, a}$, then the value does not change.
    Thus, changing to $\bP_{\pibehave, s, a}$ and expanding $q^{\pi_G}_{\cdot}(S_{\cdot}, \pibehave)$ repeatedly, we get that:
    \begin{align*}
        &= \bigE_{\trajectoryrand \sim \bP_{\pibehave, s, a}} \brackets{R_h + (1-\omega_G(S_{h+1})) q^{\pi_G}_{h+1}(S_{h+1}, \pi) \\
            &\qquad + \omega_G(S_{h+1}) (R_{h+1} + (1-\omega_G(S_{h+2})) q^{\pi_G}_{h+2}(S_{h+2}, \pi) + \omega_G(S_{h+2}) q^{\pi_G}_{h+2}(S_{h+2}, \pibehave))} \\
        \vdots \\
        &= \bigE_{\trajectoryrand \sim \bP_{\pibehave, s, a}} \brackets*{\sum_{\tau = h+1}^{H+1} \prod_{u = h+1}^{\tau-1} \omega_G(S_u) (R_{\tau-1} 
            + (1-\omega_G(S_{\tau})) q^{\pi_G}_{\tau}(S_{\tau}, \pi))} 
        \, .
    \end{align*}
    Since $\sum_{\tau = h+1}^{H+1} \prod_{u = h+1}^{\tau-1} \omega_G(S_u) R_{\tau-1} = \sum_{\tau = h+1}^{H+1} (1-\omega_G(S_{\tau}))\prod_{u = h+1}^{\tau-1} \omega_G(S_u) R_{h:\tau-1}$, 
    we have that:
    \begin{align*}
        &= \bigE_{\trajectoryrand \sim \bP_{\pibehave, s, a}} \brackets*{\sum_{\tau = h+1}^{H+1} (1-\omega_G(S_{\tau}))\prod_{u = h+1}^{\tau-1} \omega_G(S_u) (R_{h:\tau-1} 
            + q^{\pi_G}_{\tau}(S_{\tau}, \pi))} \\
        &= \bigE_{\trajectoryrand \sim \bP_{\pibehave, s, a}} \sum_{\tau = h+1}^{H+1} F_{G, \trajectoryrand, h+1}(\tau) (R_{h:\tau-1} + q^{\pi_G}_\tau(S_\tau, \pi)) \\
        &= \bigE_{\trajectoryrand \sim \bP_{\pibehave, s, a}} g_G^{\pi}(q^{\pi_G}_{\futurestages}, \trajectoryrand)
        = T_{G}^{\pi} \, q^{\pi_G}_{\futurestages} (s, a)
        \, .
    \end{align*}
    Which completes the proof.
\end{proof}

\subsection{Proof of \cref{lem:skippy bellman optimality operator backup}}
    \label{proof:skippy bellman optimality operator backup}
\begin{proof}
    We need to show that $q^{\pioptg_G}_h = T_G \, q_{\futurestages}^{\pioptg_G}$,
    for any 
    modification $G \in \Gall$.
    We make use of the fact that $q^{\pioptg_G}_h = T^{\pioptg_G} q^{\pioptg_G}_{h+1}$, and the definition of $\pioptg_G$,
    to expand $q^{\pioptg_G}_h$ as follows, for any $(s, a) \in \cS_h \times \cA$, $h \in [H]$:
    \begin{align*}
        q^{\pioptg_G}_h(s, a)
        &= T^{\pioptg_G} q^{\pioptg_G}_{h+1} \\
        &= \bigE_{\trajectoryrand \sim \bP_{\pioptg_G, s, a}} \brackets{R_h + q^{\pioptg_G}_{h+1}(S_{h+1}, \pi_G)} \\
        &= \bigE_{\trajectoryrand \sim \bP_{\pioptg_G, s, a}} \brackets{R_h + (1-\omega_G(S_{h+1})) \max_a q^{\pioptg_G}_{h+1}(S_{h+1}, a) + \omega_G(S_{h+1}) q^{\pioptg_G}_{h+1}(S_{h+1}, \pibehave)}
        \, .
    \end{align*}
    Notice that if we change the expectation to be with respect to $\bP_{\pibehave, s, a}$ instead of $\bP_{\pioptg_G, s, a}$, then the value does not change.
    Thus, changing to $\bP_{\pibehave, s, a}$ and expanding $q^{\pioptg_G}_{\cdot}(S_{\cdot}, \pibehave)$ repeatedly, we get that:
    \begin{align*}
        &= \bigE_{\trajectoryrand \sim \bP_{\pibehave, s, a}} \brackets{R_h + (1-\omega_G(S_{h+1})) \max_a q^{\pioptg_G}_{h+1}(S_{h+1}, a ) \\
            &\qquad + \omega_G(S_{h+1}) (R_{h+1} + (1-\omega_G(S_{h+2})) \max_a q^{\pioptg_G}_{h+2}(S_{h+2}, a) + \omega_G(S_{h+2}) q^{\pioptg_G}_{h+2}(S_{h+2}, \pibehave))} \\
        \vdots \\
        &= \bigE_{\trajectoryrand \sim \bP_{\pibehave, s, a}} \brackets*{\sum_{\tau = h+1}^{H+1} \prod_{u = h+1}^{\tau-1} \omega_G(S_u) (R_{\tau-1} 
            + (1-\omega_G(S_{\tau})) \max_a q^{\pioptg_G}_{\tau}(S_{\tau}, a))} 
        \, .
    \end{align*}
    Since $\sum_{\tau = h+1}^{H+1} \prod_{u = h+1}^{\tau-1} \omega_G(S_u) R_{\tau-1} = \sum_{\tau = h+1}^{H+1} (1-\omega_G(S_{\tau}))\prod_{u = h+1}^{\tau-1} \omega_G(S_u) R_{h:\tau-1}$, 
    we have that:
    \begin{align*}
        &= \bigE_{\trajectoryrand \sim \bP_{\pibehave, s, a}} \brackets*{\sum_{\tau = h+1}^{H+1} (1-\omega_G(S_{\tau}))\prod_{u = h+1}^{\tau-1} \omega_G(S_u) (R_{h:\tau-1} 
            + \max_a q^{\pioptg_G}_{\tau}(S_{\tau}, a))} \\
        &= \bigE_{\trajectoryrand \sim \bP_{\pibehave, s, a}} \sum_{\tau = h+1}^{H+1} F_{G, \trajectoryrand, h+1}(\tau) (R_{h:\tau-1} + \max_a q^{\pioptg_G}_\tau(S_\tau, a)) \\
        &= \bigE_{\trajectoryrand \sim \bP_{\pibehave, s, a}} g_G(q^{\pioptg_G}_{\futurestages}, \trajectoryrand)
        = T_{G} \, q^{\pioptg_G}_{\futurestages} (s, a)
        \, .
    \end{align*}
    Which completes the proof.
\end{proof}

\subsection{Proof of \cref{lem:Qopt guarantee}} \label{proof:Qopt guarantee}

We split the proof into two parts, showing each claim separately.

\begin{proof}
\subsubsection*{Proof of the first claim:}
    To show the first claim we make use of the following lemma from \citet{tkachuk2024trajectory}.
    \begin{lemma}[Lemma 5.1 in \citep{tkachuk2024trajectory}] \label{lem:pe Theta elements are close}
        For all $\delta \in (0, 1)$,
        with probability at least $1 - \delta/5$,
        for all $G \in \Gopt$,
        for all $h \in [H]$, 
        for all $(\theta_{s,a})_{(s,a)\in\cS_h\times\cA}$, $(\check\theta_{s,a})_{(s,a)\in\cS_h\times\cA} \in (\thetaopt_{G, h})^{\cS_h\times\cA}$,
        and for all admissible distributions $\nu = (\nu_t)_{t \in [H]}$, 
        it holds that 
        \begin{align}
            &\bigE_{(S, A) \sim \nu_h} \left[\bar q_{\theta_{S, A}}(S, A) - \bar q_{\check \theta_{S, A}}(S, A)\right]
            \le \tilde \eps 
            \defeq \conc \cdot (\bar \eps + \zeta_1)
            = \bigOt\left(\log\left(\frac{1}{\alpha}\right) \frac{\conc^{3/2} H^{5/2} d^{3/2}}{\sqrt{n}}\right) \, ,
            \label{eq:tilde eps defn}
        \end{align}
        where $\zeta_1$ is defined in \cref{eq:zeta1 defn}.
    \end{lemma}
    Which gives the following corollary, which has an absolute value inside the expectation.
    \begin{corollary} \label{lem:pe Theta elements are close in abs}
        For all $\delta \in (0, 1)$,
        with probability at least $1 - \delta/5$,
        for all $G \in \Gopt$,
        for all $h \in [H]$, 
        for all $\theta$ and $\check\theta \in \thetaopt_{G, h}$,
        and for all admissible distributions $\nu = (\nu_t)_{t \in [H]}$, 
        it holds that 
        \begin{align*}
            &\bigE_{(S, A) \sim \nu_h} \left|\bar q_{\theta}(S, A) - \bar q_{\check \theta}(S, A)\right|
            \le \tilde \eps \, .
        \end{align*}
    \end{corollary}
    \begin{proof}
        Fix a stage $h \in [H]$, modification $G \in \Gopt$, parameters $\theta, \check\theta \in \thetaopt_{G, h}$ and admissible distribution $\nu_h$.
        Define 
        \begin{align*}
            \theta_{s, a} = \begin{cases}
                \theta & \text{if } \bar q_{\theta}(s, a) - \bar q_{\check \theta}(s, a) \ge 0; \\
                \check\theta & \text{if } \bar q_{\check\theta}(s, a) - \bar q_{\theta}(s, a) > 0; 
            \end{cases} \qquad
            \check\theta_{s, a} = \begin{cases}
                \check\theta & \text{if } \bar q_{\theta}(s, a) - \bar q_{\check \theta}(s, a) \ge 0; \\
                \theta & \text{if } \bar q_{\check\theta}(s, a) - \bar q_{\theta}(s, a) > 0 \, . 
            \end{cases}
        \end{align*}
        Clearly, this implies that
        \begin{align*}
            \bar q_{\theta_{s, a}}(s, a) - \bar q_{\check \theta_{s, a}}(s, a) = \left|\bar q_{\theta}(s, a) - \bar q_{\check \theta}(s, a) \right| \qquad \forall (s, a) \in \cS_h \times \cA  \, ,
        \end{align*}
        which gives the desired result since by \cref{lem:pe Theta elements are close}, 
        $\bE_{(S, A) \sim \nu_h} \brackets{\bar q_{\theta_{S, A}}(S, A) - \bar q_{\check \theta_{S, A}}(S, A)} \le \tilde \eps$.
    \end{proof}
    Combining \cref{lem:pievalG in Theta} and \cref{lem:pe Theta elements are close in abs} with a union bound gives that with probability at least $1 -2\delta/5$,
    for all $G \in \Gopt$, $h \in [H]$ and $\theta \in \thetaopt_{G, h}$:
    \begin{align*}
        &\bigE_{(S, A) \sim \nu_h} \left|\bar q_{\theta}(S, A) - \bar q_{\thetatrue_h(\pioptg_G)}(S, A)\right|
        \le \tilde \eps  
        \, .
    \end{align*}
    The above equation proves the first claim of the lemma,
    since for $q^G \in \Qopt$, there exists a $\theta \in \thetaopt_{G, h}$ such that $q^G_h = \bar q_\theta$, and $\bar q_{\thetatrue_h(\pieval_G)} = q^{\pioptg_G}_h$ by linear $q^\pi$-realizability of $\pioptg_G$.

\subsubsection*{Proof of the second claim:} \label{proof:second claim of Qopt guarantee}
    It will suffice to show that $\trueG \in \Gopt$, since each $q^G \in \Qopt$ is defined based on some $G \in \Gopt$.
    Lemma D.3 in \citet{tkachuk2024trajectory} gave a weaker bound than ours, in the sense that their $\bar \eps$ in the definition of $\Gopt$ was larger by a factor of $\sqrt{\conc d}$.
    Next, we prove the claim and show how we improved upon their result.

    The following lemma shows that the data dependent condition in the definition of $\Gopt$ concentrates around its expectation.
    \begin{lemma}[Lemma I.1 in \citep{tkachuk2024trajectory}]
        \label{lem:optG concentrates}
        For all $\delta \in (0, 1)$,
        with probability at least $1 - \delta/5$, 
        for all $G \in \bbG$,
        and for all $h \in [H]$, 
        \begin{align}
            &\abs*{\bigE_{(S, A) \sim \mu_h} \left[\max_{\theta \in \Theta_{G, h}} \bar q_{\theta}(S, A) - \min_{\theta \in \Theta_{G, h}} \bar q_{\theta}(S, A)\right]
                - \frac{1}{n} \sum_{j \in [n]} \left(\max_{\theta \in \Theta_{G, h}} \bar q_{\theta}(S_h^j, A_h^j) - \min_{\theta \in \Theta_{G, h}} \bar q_{\theta}(S_h^j, A_h^j)\right)} \nonumber \\
            &\le \zeta_1
                \defeq \frac{2H}{\sqrt{n}} \sqrt{dH^2d_0\log\left(1 + 96 \sqrt{n} \sqrt{2d} H^2 \featurebound \thetabound \alpha^{-1} \sqrt{n}\featurebound\modthetabound / (H^{3/2}d)\right) 
                    + \log\left(\frac{20H}{\delta}\right)}
            \, .
            \label{eq:zeta1 defn}
        \end{align}
    \end{lemma}
    Thus, it will sufficient to bound $\bE_{(S, A) \sim \mu_h} \brackets{\max_{\theta \in \Theta_{\trueG, h}} \bar q_{\theta}(S, A) - \min_{\theta \in \Theta_{\trueG, h}} \bar q_{\theta}(S, A)}$.
    First we state a helpful lemma.
    \begin{lemma}[Lemma G.1 in \citep{tkachuk2024trajectory}]
        \label{lem:sum of elliptical terms bound}
        Let
        \begin{align}
            \sqrt{\lambda}
            &\defeq H^{3/2} d / \modthetabound
            \, ,
            \label{eq:lambda defn}
            \\ 
            \beta 
            &\defeq \sqrt{\lambda} \modthetabound + \bar \beta
            \label{eq:beta defn}
            \\
            \bar \beta
            &\defeq 2H \sqrt{2dH(d_0+1)\log\left(1 + 28\sqrt{2d} H^2 \thetabound \modthetabound \featurebound \alpha^{-1}\right) + d \log (\lambda + n \featurebound^2/d) - d \log(\lambda) 
                + \log\left(\frac{10H}{\delta}\right)} \nonumber
            \, .
        \end{align}
        For any $\delta \in (0, 1)$,
        with probability at least $1 - \delta/5$
        for all $h \in [H], 
        (s, a) \in \cS_h \times \cA, 
        \theta_h \in \Theta_{\trueG, h}$, 
        it holds that
        \begin{align}
            \abs{\bar q_{\theta_h}(s, a) - q^{\pioptg_{\trueG}}(s, a)} 
            \le 2\beta\bigE_{\trajectoryrand \sim \bP_{\bar \pi, s, a}} \sum_{t = h}^{H} \min\braces{1, \norm{\phi(S_t, A_t)}_{\hat X_t^{-1}}} \, ,
            \label{eq:q close inductive hypothesis}
        \end{align}   
        where, for all $(s', a') \in \cS \times \cA$,
        \begin{align}
            \bar \pi(a'|s') 
            &= \pibehave(a'|s') \omega_{\trueG}(s') + \I{\argmax_{a'' \in \cA} g^{\bar \pi}(s', a'') = a'} (1 - \omega_{\trueG}(s')) \, ,
            \label{eq:optpi elliptical}
        \end{align}
        and $g^{\bar \pi}$ is a state-action value function of policy $\bar \pi$ (similar to $q^{\bar \pi}$), 
        except in the alternative MDP that has the same state and action spaces, and transition distributions as the original MDP under consideration, but with a reward function modified as follows. 
        For all $(s', a') \in \cS \times \cA$, the reward in this alternative MDP is deterministically $\min\braces{1, \norm{\phi(s', a')}_{\hat X_h^{-1}}}$. 
        In particular for any $h' \in [H], (s', a') \in \cS_{h'} \times \cA$,
        \begin{align}
            g^{\bar \pi}(s', a') = \bigE_{\trajectoryrand \sim \bP_{\bar \pi, s', a'}} \sum_{t = h'}^{H} \min\braces{1, \norm{\phi(S_t, A_t)}_{\hat X_t^{-1}}} \, .
            \label{eq:gpi definition}
        \end{align}
    \end{lemma}
    Let $\bar \pi$ be as defined in \cref{lem:sum of elliptical terms bound}.
    Define a new policy $\check \pi_h$ as
    \begin{align*}
        \check \pi_h(a|s)
        = \begin{cases}
            \pibehave(a|s), & \text{if } \stage(s) \le h; \\
            \bar \pi(a|s), & \text{if } \stage(s) > h;  
        \end{cases}
        \quad \text{for all } (s, a) \in \cS \times \cA \, ,
    \end{align*}
    where $\stage(s)$ gives the stage index of state $s$.
    Then, by \cref{lem:sum of elliptical terms bound}, with probability at least $1 - \delta/5$, for all $h \in [H]$,
    \begin{align}
        &\bigE_{(S_h, A_h) \sim \mu_h} \left[\max_{\theta \in \Theta_{\trueG, h}} \bar q_\theta(S_h, A_h) - \min_{\theta \in \Theta_{\trueG, h}} \bar q_\theta(S_h, A_h)\right] \nonumber \\
        &= \bigE_{(S_h, A_h) \sim \mu_h} \left[\max_{\theta \in \Theta_{\trueG, h}} \bar q_\theta(S_h, A_h) - v^{\pioptg_{\trueG}}(S_h) + v^{\pioptg_{\trueG}}(S_h) - \min_{\theta \in \Theta_{\trueG, h}} \bar q_\theta(S_h, A_h)\right] \nonumber \\
        &\le 4\beta\bigE_{\trajectoryrand \sim \bP_{\check \pi, \startstate}} \sum_{t = h}^{H} \min\braces{1, \norm{\phi(S_t, A_t)}_{\hat X_t^{-1}}} \nonumber
        \, .
    \end{align}
    So far the analysis has followed the same steps as in \citet{tkachuk2024trajectory}.
    However, we will now deviate from their analysis to get a tighter bound.
    To do so we first use the linearity of expectation and Jensen's inequality to get that
    \begin{align*}
        4\beta\bigE_{\trajectoryrand \sim \bP_{\check \pi, \startstate}} \sum_{t = h}^{H} \min\braces{1, \norm{\phi(S_t, A_t)}_{\hat X_t^{-1}}} 
        \le 4\beta \sum_{t = h}^{H} \sqrt{\bigE_{\trajectoryrand \sim \bP_{\check \pi, \startstate}} \min\braces{1, \norm{\phi(S_t, A_t)}^2_{\hat X_t^{-1}}}} 
        \, .
    \end{align*}
    Notice that for any $t \in [h:H]$, 
    $\bigE_{\trajectoryrand \sim \bP_{\check \pi, \startstate}} \min\braces{1, \norm{\phi(S_t, A_t)}^2_{\hat X_t^{-1}}}$
    is an admissible distribution.
    By \cref{lem:change of measure}, 
    \begin{align*}
        4\beta \sum_{t = h}^{H} \sqrt{\bigE_{\trajectoryrand \sim \bP_{\check \pi, \startstate}} \min\braces{1, \norm{\phi(S_t, A_t)}^2_{\hat X_t^{-1}}}} 
        \le 4\sqrt{\conc}\beta \sum_{t = h}^{H} \sqrt{\bigE_{(S_t, A_t) \sim \mu_t} \min\braces{1, \norm{\phi(S_t, A_t)}^2_{\hat X_t^{-1}}}} 
        \, .
    \end{align*}
    Since we applied Jensen's inequality first, and then used \cref{lem:change of measure}, 
    the concentrability coefficient $\conc$ appears under the square root, which gives us an improvement over \citet{tkachuk2024trajectory} by a factor of $\sqrt{\conc}$.

    The following lemma shows that the terms under the square root concentrate at a rate of $\bigOt(d/n)$, 
    which improves upon the $\bigOt(d^2/n)$ rate in \citet{tkachuk2024trajectory}.
    The improvement comes from an elliptical potential and exponential supermartingale concentration argument.

    \begin{lemma} \label{lem:elliptical terms concentrate}
        For any $\delta \in (0, 1)$,
        with probability at least $1 - \delta/5$,
        for all $h \in [H]$, 
        \begin{align*}
            \bigE_{(S_h, A_h) \sim \mu_h} \min\braces*{1, \norm{\phi(S_h, A_h)}^2_{\hat X_h^{-1}}}
            \le \check \eps
            = \bigOt\paren{d/n}
            \, ,
        \end{align*}
        where $\check \eps$ is defined in \cref{eq:check eps defn}.
    \end{lemma}
    \begin{proof}
        Fix any $h \in [H]$.
        Let 
        \begin{align*}
            \cH_t \defeq \sigma\braces{(S_h^j, A_h^j)_{j \in [t]}}
        \end{align*}
        be the sigma algebra generated by the first $t$ samples at stage $h$.
        Let 
        \begin{align*}
            V_t \defeq \lambda I + \sum_{j \in [t]} \phi(S_h^j, A_h^j) \phi(S_h^j, A_h^j)^\top
        \end{align*}    
        be the unnormalized regularized empirical covariance matrix based on the first $t$ samples at stage $h$.
        For any $t \in [n+1]$, let 
        \begin{align*}
            q_t \defeq \norm{\phi(S_h^t, A_h^t)}_{V_{t-1}^{-1}}^2,
            \quad
            Z_t \defeq \min \braces{1, q_t},
            \quad
            u_t \defeq \bE[Z_t | \cH_{t-1}] 
            = \bE \brackets*{\min\braces{1, \norm{\phi(S_h^t, A_h^t)}_{V_{t-1}^{-1}}^2} \mid \cH_{t-1}}
            \, ,
        \end{align*}
        where $S_h^{n+1}, A_h^{n+1}$ is an independent sample from $\mu_h$.
        Since $V_n = \hat X_h$, we have that
        \begin{align*}
            \bigE_{(S_h, A_h) \sim \mu_h} \min\braces*{1, \norm{\phi(S_h, A_h)}^2_{\hat X_h^{-1}}}
            &= \bigE \brackets*{\min\braces*{1, \norm{\phi(S_h^{n+1}, A_h^{n+1})}_{V_n^{-1}}^2} \mid \cH_n}
            = u_{n+1}
            \, .
        \end{align*}
        Thus, our goal is to bound $u_{n+1}$ with high probability.
        The proof will proceed with the following steps:
        \begin{enumerate}
            \item We show that $u_{n+1} \le \tfrac{1}{n} \sum_{t \in [n]} u_t$, since $u_t$ is non-increasing in $t$ due to $V_t \succcurlyeq V_{t-1}$.
            \item We show that $\sum_{t \in [n]} Z_t \le \bigOt(d)$ deterministically by using the elliptical potential lemma.
            \item We show that with high probability, $\sum_{t \in [n]} u_t \le 2\sum_{t \in [n]} Z_t + 2\log(1/\delta')$ by using a concentration argument.
            \item Combining the above steps gives that with high probability, $u_{n+1} \le \bigOt(d/n)$.
        \end{enumerate}

        \noindent
        \textbf{Step 1:}
        Since $V_t \succcurlyeq V_{t-1}$ for all $t \in [n]$, we have that $\norm{\phi(S, A)}_{V_t^{-1}}^2 \le \norm{\phi(S, A)}_{V_{t-1}^{-1}}^2$ for all $(S, A) \in \cS_h \times \cA$.
        For an independent sample $S, A \sim \mu_h$ we have that
        \begin{align*}
            u_{t+1} 
            = \bE \brackets*{\min\braces{1, \norm{\phi(S, A)}_{V_{t}^{-1}}^2} \mid \cH_{t}}
            &\le \bE \brackets*{\min\braces{1, \norm{\phi(S, A)}_{V_{t-1}^{-1}}^2} \mid \cH_{t}} \\
            &= \bE \brackets*{\min\braces{1, \norm{\phi(S, A)}_{V_{t-1}^{-1}}^2} \mid \cH_{t-1}}
            = u_t
            \, .
        \end{align*}
        Thus, $u_t$ is non-increasing in $t$, which gives that $u_{n+1} \le \frac{1}{n} \sum_{t \in [n]} u_t$.

        \noindent
        \textbf{Step 2:}
        By the elliptical potential lemma (see, e.g., Lemma 19.4 in \citep{lattimore2020bandit}), we have that deterministically,
        \begin{align*}
            \sum_{t \in [n]} Z_t 
            \le 2d \log \paren*{\frac{\trace(V_0) + n \featurebound^2}{d \det(V_0)^{1/d}}}
            = 2d \log \paren*{\frac{\lambda d + n \featurebound^2}{\lambda d}}
            \, .
        \end{align*}

        \noindent
        \textbf{Step 3:}
        For any $z \in [0, 1]$, it holds that $e^{-z} \le 1 - cz \le e^{-cz}$ with $c \defeq 1 - e^{-1}$.
        Since $Z_t \in [0, 1]$ for all $t \in [n]$, we have that
        \begin{align*}
            \bE \brackets*{e^{-Z_t} \mid \cH_{t-1}}
            \le \bE \brackets*{1 - cZ_t \mid \cH_{t-1}}
            = 1 - c u_t
            \le e^{-c u_t}
            \, .
        \end{align*}
        This implies that 
        \begin{align*}
            \bE \brackets*{\exp\paren*{c u_t - Z_t} \mid \cH_{t-1}}
            \le 1
            \, .    
        \end{align*}
        Let 
        \begin{align*}
            B_t \defeq \exp(c u_t - Z_t), 
            \quad
            M_t \defeq \prod_{j \in [t]} B_j = \exp\paren*{c \sum_{j \in [t]} u_j - \sum_{j \in [t]} Z_j} \, .
        \end{align*}
        Since $M_{t-1}$ is $\cH_{t-1}$ measurable, we have that
        \begin{align*}
            \bE \brackets*{M_t \mid \cH_{t-1}}
            = \bE \brackets*{M_{t-1} B_t \mid \cH_{t-1}}
            = M_{t-1} \bE \brackets*{B_t \mid \cH_{t-1}}
            \le M_{t-1}
            \, .
        \end{align*}
        Thus, $M_t$ is a nonnegative supermartingale.
        By the supermartingale property we have that $\bE[M_n] \le \bE[M_0] = 1$.
        Markov's inequality states that
        \begin{align*}
            \bP\paren*{M_n \ge 1/\delta'}
            \le \delta' \bE[M_n]
            \le \delta'
            \, .
        \end{align*}
        Thus, with probability at least $1 - \delta'$, 
        \begin{align*}
            c \sum_{t \in [n]} u_t - \sum_{t \in [n]} Z_t
            = \log(M_n)
            \le \log(1/\delta')
            \, .
        \end{align*}
        Rearranging gives that with probability at least $1 - \delta'$,
        \begin{align*}
            \sum_{t \in [n]} u_t
            \le \frac{1}{c} \sum_{t \in [n]} Z_t + \frac{1}{c} \log(1/\delta')
            \, .
        \end{align*}    

        \noindent
        \textbf{Step 4:}
        Combining the above steps gives that with probability at least $1 - \delta'$,
        \begin{align*}
            u_{n+1}
            \le \frac{1}{n} \sum_{t \in [n]} u_t
            \le \frac{1}{(1 - e^{-1}) n} \paren*{\sum_{t \in [n]} Z_t + \log(1/\delta')}
            \le \frac{4}{n} \paren*{d \log \paren*{\frac{\lambda d + n \featurebound^2}{\lambda d}} + \log(1/\delta')}
            = \bigOt(d/n)
            \, .
        \end{align*}    

        To conclude, we can set $\delta' = \delta/(5H)$ and apply a union bound over all $h \in [H]$ to get that with probability at least $1 - \delta/5$, for all $h \in [H]$,
        \begin{align}
            \bigE_{(S_h, A_h) \sim \mu_h} \min\braces*{1, \norm{\phi(S_h, A_h)}^2_{\hat X_h^{-1}}}
            &\le \frac{4}{n} \paren*{d \log \paren*{\frac{\lambda d + n \featurebound^2}{\lambda d}} + \log(5H/\delta)} \nonumber \\
            &\eqdef \check \eps
            = \bigOt\paren{d/n}
            \, .
            \label{eq:check eps defn}
        \end{align}    
        Which is the desired result.
    \end{proof}

   By \cref{lem:elliptical terms concentrate}, 
   with probability at least $1 - \delta/5$,
   \begin{align*}
        4\sqrt{\conc}\beta \sum_{t = h}^{H} \sqrt{\bigE_{(S_t, A_t) \sim \mu_t} \min\braces{1, \norm{\phi(S_t, A_t)}^2_{\hat X_t^{-1}}}} 
        \le 4\sqrt{\conc}\beta \sum_{t = h}^{H} \sqrt{\check \eps} 
   \end{align*}
    Putting the results from the proof of the second claim together, and using a union bound, we get that with probability at least $1 - 3\delta/5$,
    for all $h \in [H]$,
    \begin{align}
        \frac{1}{n} \sum_{j \in [n]} \left(\max_{\theta \in \Theta_{G, h}} \bar q_{\theta}(S_h^j, A_h^j) - \min_{\theta \in \Theta_{G, h}} \bar q_{\theta}(S_h^j, A_h^j)\right)
        \le \zeta_1 + 4\sqrt{\conc} H \beta \sqrt{\check \eps}
        \eqdef \bar \eps
        \, .
        \label{eq:bar eps defn}
    \end{align}
    
\noindent Applying a union bound over the failure events of the two claims completes the proof of \cref{lem:Qopt guarantee}.
\end{proof}

\subsection{Proof of \cref{lem:Qeval guarantee}} \label{proof:Qeval guarantee}
\begin{proof}
    The proof follows the same steps as the proof of \cref{lem:Qopt guarantee}, 
    except $\Geval$, $\thetaeval_{G, h}$ and $\pieval_G$, are used instead of $\Gopt$, $\thetaopt_{G, h}$ and $\pioptg_G$.
    Also, the lemmas need to be changed to hold with probability at least $1 - \delta/10$ instead of $1 - \delta/5$.
\end{proof}

\section{Useful Results} \label{ss:other useful results}

\begin{lemma}[Hoeffding's Inequality (Theorem 2 in \citep{hoeffding1994probability})] \label{lem:hoeffdings inequality}
    Let $(X_i)_{i \in \bN}$ be independent random variables such that $X_i \in [a, b]$ for some $a, b \in \bR$, and let $S_n = \frac{1}{n} \sum_{i=1}^n X_i$. 
    Then, with probability at least $1 - \zeta$ it holds that
    \begin{align*}
        \abs{\bigE S_n - S_n}
        \le \frac{(b-a)}{\sqrt{n}} \sqrt{\log\left(\frac{2}{\zeta}\right)} \, .
    \end{align*}
\end{lemma}

\begin{lemma}[Covering number of the Euclidean ball] \label{lem:cover number ball}
    Let $a > 0, \eps > 0, d \ge 1$, and $\cB_d(a)=\{x \in \bR^d: \norm{x}_2 \le a\}$ denote the $d$-dimensional Euclidean ball of radius $a$ centered at the origin.
    The covering number of $\cB_d(a)$ is upper bounded by $\left(1 + \frac{2a}{\eps}\right)^d$. 
\end{lemma}
\begin{proof}
    Same as the proof of Corollary 4.2.13 in \citep{vershynin2018high} with $\cB(1)$ replaced with $\cB(a)$.
\end{proof}

\begin{lemma}[Performance Difference Lemma (Lemma 3.2 in \citep{cai2020provably})] \label{lem:performance difference lemma}
    For any policies $\pi, \bar \pi$, it holds that  %
    \begin{align*}
        v^{\pi}(s_1) - v^{\bar \pi}(s_1)
        = \sum_{h=1}^H \bigE_{(S_h, A_h) \sim \bPmarg{h}_{\pi, s_1}} \left(q^{\bar \pi}(S_h, A_h) - v^{\bar \pi}(S_h)\right) \, .
    \end{align*}
\end{lemma}

\end{document}